\documentclass{article}

\PassOptionsToPackage{numbers,compress}{natbib}

\usepackage[preprint]{neurips_2026}

\usepackage{microtype}
\usepackage{hyperref}
\usepackage{url}
\usepackage{booktabs}

\usepackage[utf8]{inputenc} %
\usepackage[T1]{fontenc}    %
\usepackage{hyperref}       %
\usepackage{url}            %
\usepackage{booktabs}       %
\usepackage{amsfonts}       %
\usepackage{nicefrac}       %
\usepackage{microtype}      %
\usepackage{xcolor}         %
\usepackage{comment}
\usepackage{amsmath}
\usepackage{mathtools}

\usepackage{bm}
\usepackage{amsthm}
\renewcommand{\vec}[1]{\bm{#1}}
\newcommand{\sumT}{\sum_{t=1}^T} 
\usepackage{algorithm}
\usepackage{algpseudocode}
\newcommand{\grd}{\nabla}
\newcommand{\expec}[2]{\mathop{\mathbb{E}}_{#1}\left[{#2}\right]}
\newcommand{\cQ}{\mathcal{Q}}
\newcommand{\comseq}{\{\vec u_t\}_{t=1}^T}

\newcommand{\dtp}[2]{\langle {#1}, {#2} \rangle}
\newcommand{\X}{\mathcal{X}}

\DeclareMathOperator*{\argmin}{argmin}

\newcommand{\ti}{\tilde}
\newcommand{\R}{\mathcal{R}}
\newcommand{\bigo}{\mathcal{O}}
\newcommand{\Lap}{\mathrm{Lap}}
\newcommand{\Rind}{\R^{\mathbf{1\{\cdot\}}}}
\renewcommand{\Rind}{\R^{\mathbf{1}}}
\newcommand{\RindH}[1]{\R^{\mathbf{1},{{#1}}}}

\usepackage[capitalise]{cleveref}
\usepackage{thmtools}
\declaretheorem[name=Theorem]{theorem}
\declaretheorem[name=Lemma, sibling=theorem]{lemma}
\declaretheorem[name=Proposition, sibling=theorem]{proposition}
\declaretheorem[name=Assumption, sibling=theorem]{assumption}

\declaretheorem[name=Corollary, sibling=theorem]{corollary}

\newcommand{\E}{\mathbb{E}}

\newcommand{\one}{\mathbf{1}}
\newcommand{\TV}[1]{\left\lVert #1 \right\rVert_{\mathrm{TV}}}

\newcommand{\cP}{\mathcal{P}}
\newcommand{\cH}{\mathcal{H}}
\newcommand{\cI}{\mathcal{I}}
\newcommand{\cG}{\mathcal{G}}

\newcommand{\cE}{\mathcal{E}}
\newcommand{\cD}{\mathcal{D}}
\newcommand{\cU}{\mathcal{U}}

\newcommand{\ind}{\mathbf{1}\{\vec x_t \neq \vec x_{t-1}\}}
\newcommand{\ip}[2]{\left\langle #1, #2 \right\rangle}
\usepackage{enumitem}

\title{Dynamic Regret in Online Convex Optimization with Indicator Switching Costs}

\author{
  Naram Mhaisen \qquad George Iosifidis \\
  Faculty of Electrical Engineering, Mathematics and Computer Science \\
  TU Delft, Netherlands
}
\begin{document}
\maketitle

\begin{abstract}
We study dynamic regret in online convex optimization with an \emph{indicator switching cost}: a fixed penalty incurred whenever two consecutive decisions differ. This captures startup overheads such as server activation, model deployment, and cache updates, and on a bounded domain it recovers norm-based movement costs as a special case. Existing guarantees for indicator costs handle only static comparators. We show that a direct extension of these techniques to dynamic regret provably fails, motivating a different approach. We propose a  meta-learning framework: a set of randomized lazy FTRL base learners restarted at dyadic time scales, aggregated by a movement-aware master that mixes their proposal densities and samples actions via maximal coupling of consecutive mixtures. The resulting algorithm satisfies, in expectation, $\mathcal{R}^{\mathbf{1}}_T \le \tilde{\mathcal{O}}(\min\{\sqrt{T(S_T{+}1)},T^{2/3}(P_T+1)^{1/3}\})$, where $\mathcal{R}^{\mathbf{1}}_T$ is the dynamic regret plus the cumulative indicator switching cost, $S_T$ counts comparator switches, and $P_T$ is the comparator path length. The bound holds simultaneously for all sequences and requires no prior knowledge of $S_T$ or $P_T$: it is minimax-optimal (up to logarithmic factors) for tracking piecewise-constant comparators, and also captures frequently moving comparators with small total path length.
\end{abstract}

\section{Introduction}
We study Online Convex Optimization (OCO) in non-stationary environments. At each round $t \in [T]$, a learner selects an action $\vec x_t \in \mathcal X$ \emph{before} observing the convex loss $f_t(\cdot)$, and incurs $f_t(\vec x_t)$. A central performance criterion in this setting is \emph{dynamic regret}, which compares the learner against a time-varying comparator sequence $\comseq$. To capture stability requirements, dynamic regret is often augmented with a \emph{switching cost} $m_t(\vec x_t,\vec x_{t-1})$ that penalizes changes between consecutive actions. The augmented formulation, known as \emph{smoothed} OCO (SOCO), models the central tradeoff between responsiveness to non-stationarity and stability of the decision sequence. While this tradeoff is well understood for switching costs that are norm-based, the regime in which any reconfiguration incurs an overhead independent of the magnitude of the change remains underexplored.

In many practical problems, the switching cost is modeled as an \emph{indicator switching cost}
\citep{cesa2013online, jaghargh2019consistent, sherman2021lazy}: $\lambda_t \,\mathbf 1\{\vec x_t \neq \vec x_{t-1}\},$ which charges a (possibly time-varying) penalty $\lambda_t$ whenever $\vec x_t \neq \vec x_{t-1}$, regardless of how large the change is. Such costs arise whenever any reconfiguration triggers a one-shot overhead: server activation in cloud scheduling, cache updates in networks, model redeployment in production ML pipelines, retraining-trigger costs in continual learning, and the per-iteration privacy budget in differentially private learning~\citep{agarwal2023differentially}. The indicator switching cost is, in fact, the \emph{more demanding} primitive. Setting $\lambda_t = \|\vec x_t - \vec x_{t-1}\|$ recovers the norm-based switching cost as a special case, and on a bounded domain $\|\vec x_t - \vec x_{t-1}\| \le D \,\mathbf 1\{\vec x_t \neq \vec x_{t-1}\}$ where $D \doteq \mathrm{diam}(\mathcal X)$ is the diameter of the set $\mathcal{X}$, so any algorithm that controls the indicator switching cost automatically controls the norm-based one (up to the factor $D$), but not vice versa. In this sense indicator switching costs subsume norm-based ones and impose a strictly stronger notion of stability.

In Learning with Expert Advice (LEA), where the decision set is the
simplex and losses are linear, indicator switching costs admit a useful
reduction. Under a maximal coupling%
\footnote{Specifically, under such a coupling,
$\Pr(\vec x_t \neq \vec x_{t-1}) = \tfrac12\|\vec x_t-\vec x_{t-1}\|_1$; see, e.g., \citep[Sec.~5.2]{daniely2019competitive}.},
the expected number of switches is exactly the $\ell_1$ movement of the played distributions, so the problem reduces to controlling $\ell_1$ norm on the simplex, and both static \citep{altschuler2018online} and dynamic\footnote{In LEA, Dynamic regret is known as tracking regret. Also, strongly adaptive regret implies tracking regret \citep[Sec. 3]{daniely2015strongly}.} \citep{{daniely2019competitive}} regret bounds follow
from standard algorithms such as exponentiated gradient and fixed-share.
In general OCO, a line of work has nevertheless made progress on this
harder regime via the FTRL framework, obtaining minimax-optimal
\emph{static} regret under indicator switching costs
\citep{shermanArxiv,agarwal2023differentially}. These remain the only
available results in OCO, because the other main framework, Online Mirror Descent (OMD), defines each iterate recursively from the previous one via a Bregman divergence term (see, e.g., \citep[Sec. 6]{mcmahan-survey17}). This recursive dependence of each action on its predecessor precludes a closed-form density that can be evaluated pointwise, which is necessary for the maximal coupling used to control indicator switches.

The dynamic-regret case has thus remained open, and the difficulty traces back to the FTRL framework itself. Even without switching costs, vanilla FTRL is
suboptimal for dynamic regret \citep{jacobsen2022parameter}; prior work
has recovered dynamic-regret guarantees through enhancements such as history pruning \citep{pruning-icml,mhaisen2026partially}, but these methods proceed by
linearizing the FTRL objective, after which the iterate also becomes a recursive update, inheriting the same difficulty as OMD in deriving the density. Our
first contribution shows that even when assisted with perturbations and
fixed restarts (which can be viewed as the analogue of history pruning
but without linearization) FTRL still fails to deliver (optimal) sublinear
dynamic regret. A different architecture beyond a single FTRL run is therefore required.

The natural question is thus: \emph{can one obtain dynamic regret
guarantees in OCO when the learner pays indicator switching costs, and
if so, against which comparator classes?} We answer this in the
affirmative for the two canonical comparator classes in the dynamic
regret literature: piecewise-constant comparators that change at most
$S_T$ times, and comparators whose total variation is bounded by $P_T$.
As we detail later in the contributions, the bounded-variation class
is the larger of the two, and it further admits comparators that change
at every round while moving only slightly, sequences against which the
switch count $S_T$ is uninformative. Against the first class, we obtain
order-optimal dynamic regret of $\tilde{\mathcal{O}}(\sqrt{(S_T+1)\,T})$,
matching the minimax rate up to logarithmic factors. Against the second,
the same algorithm attains $\tilde{\mathcal{O}}(T^{2/3} P_T^{1/3})$.
A single comparator-oblivious algorithm thus handles both canonical
forms of non-stationarity simultaneously. The approach section below details how this is achieved through a multiscale meta-learner over restarted lazy randomized FTRL experts.

\subsection{Contributions}
We first state the standing assumptions used throughout the paper.
\begin{assumption}
\label{assump:domain}
The set $\mathcal{X} \subset \mathbb{R}^d$ is convex and bounded with
diameter $D \doteq \sup_{\vec{x},\vec{y}\in\mathcal{X}}\|\vec{x}-\vec{y}\|$
and radius $R \doteq \sup_{\vec{x}\in\mathcal{X}}\|\vec{x}\|$.
\end{assumption}

\begin{assumption}
\label{assump:losses}
The loss $f_t : \mathcal{X} \to \mathbb{R}$
is twice differentiable, convex, $G$-Lipschitz, and $\beta$-smooth; i.e.,
\[
    \|\nabla f_t(\vec{x})\| \le G,
    \qquad
    \nabla^2 f_t(\vec{x}) \preceq \beta I,
    \qquad \forall t\in[T], \vec{x} \in \mathcal{X}.
\]
\end{assumption}
\begin{assumption}
\label{assump:switching}
The switching coefficients $\{\lambda_t\}_{t=2}^T$ satisfy $0<\lambda_t\le \lambda$
for a known constant $\lambda>0$.
\end{assumption}

\begin{assumption}
\label{assump:bounded-function}
There exists $M_f>0$ such that $0 \le f_t(\vec{x}) \le M_f,\quad \forall t\in[T], \vec{x} \in \mathcal{X}$.
\end{assumption}

Our performance criterion is the \emph{dynamic regret with indicator switching cost} $\Rind_T$, defined as:
\begin{align*}
    \Rind_T(\comseq)
    \doteq
    \sum_{t=1}^T \bigl(f_t(\vec x_t)-f_t(\vec u_t)\bigr)
    + \sum_{t=2}^T \lambda_t\,\mathbf{1}\{\vec x_t \neq \vec x_{t-1}\},
\end{align*}
When clear from context, we drop $\comseq$ from the argument. $\Rind_T$ compares the learner against a time-varying comparator sequence
$\comseq=(\vec u_1,\ldots,\vec u_T)$ while charging the learner
for its own switches. We measure the complexity of the comparator sequence through two canonical quantities:
\begin{align*}
    S_T(\comseq)
    \doteq \sum_{t=2}^T \mathbf{1}\{\vec u_t \neq \vec u_{t-1}\},
    \qquad
    P_T(\comseq)
    \doteq \sum_{t=2}^T \|\vec u_t-\vec u_{t-1}\|,
\end{align*}
These measures induce two comparator classes,
\[
    \mathcal U_S(s) \doteq \{\comseq : S_T(\comseq) \le s\},
    \qquad
    \mathcal U_P(p) \doteq \{\comseq : P_T(\comseq) \le p\}.
\]
These classes capture different forms of non-stationarity. $\mathcal U_S(s)$ allows the comparator to change at most $s$ times, regardless of the magnitude. $\mathcal U_P(p)$ instead places no limit on how often the comparator changes, only on the total magnitude of changes, which cannot exceed $p$. Since any sequence with $s$ changes has total movement at most $D s$, we have $\mathcal U_S(s) \subseteq \mathcal U_P(Ds)$; but $\mathcal U_P$ also contains sequences that, e.g., move at every round by small amounts, for which $S_T$ becomes uninformative. Our guarantees are therefore stated
simultaneously for every sequence as functions of $S_T$ and $P_T$, with the algorithm itself oblivious to which complexity is smaller. The contributions are listed below.

\textbf{An impossibility for randomized FTRL.}
We first show that the existing machinery for indicator switching costs, known as FPRLL \cite{sherman2021lazy, shermanArxiv,agarwal2023differentially}\footnote{These works use different names for essentially the same algorithm: FTRL with perturbations and a barrier; see the Addendum  section of \cite{shermanArxiv}.}, cannot deliver optimal dynamic regret:
\begin{enumerate}[label=(\alph*), leftmargin=*, align=left]
  \item[(a)]\!\!\!\! Without restarts, FPRLL suffers $\Omega(T)$ dynamic regret even on $\mathcal U_S(1)$.
  \item[(b)]\!\!\!\! With any fixed restart period $k$, FPRLL$(k)$ suffers $\Omega\bigl(T^{2/3}\tau^{1/3}/\log T\bigr)$ dynamic regret on $\mathcal U_S(\tau)$.
\end{enumerate}
Note that the same lower bounds remain valid for the larger class $\mathcal U_P(\cdot)$.
\textbf{Strongly adaptive regret.}
We introduce a comparator-oblivious algorithm that combines restarted FPRLL base learners at all dyadic time scales, aggregated by the discounted-normal predictor of \citep{daniely2019competitive}. For every interval $I = [s,e] \subseteq [T]$ of length $L$,
\begin{align*}
  \Rind_I(\vec u)\doteq\E\Biggl[\sum_{t \in I} \bigl(f_t(\vec x_t) - f_t(\vec u)\bigr)
  + \sum_{t=s+1}^{e} \lambda_t \,\mathbf 1\{\vec x_t \neq \vec x_{t-1}\}\Biggr]
  \;\le\; \tilde{\mathcal{O}}(\sqrt L)
  \tag*{(\Cref{thm:reduction-geometric})}
\end{align*}
uniformly in $\vec u \in \mathcal X$ and in the interval $I$. To our knowledge, this is the first strongly adaptive regret guarantee under indicator switching costs for OCO, and may be of independent interest.

\textbf{Optimal dynamic regret for piecewise-constant comparators.}
A standard reduction along the time segmentation of $\comseq$ converts the strongly adaptive bound above into a dynamic regret guarantee: 
\begin{align}
  \expec{}{\Rind_T} \;\le\; \tilde{\mathcal O}\bigl(\sqrt{(S_T+1)\,T}\bigr)
  \tag*{(\Cref{cor:dr-pwc})},
\end{align}
without prior knowledge of $S_T$. This matches, up to logarithmic factors, the minimax rate for tracking moving comparators \cite{zhang2018adaptive}, despite the learner incurring an additional cost for its indicator switches.
\textbf{Path-length refinement and a unified bound.}
The same algorithm is adaptive to path measure:
\begin{align*}
  \E\bigl[\Rind_T\bigr]
  \;\le\; \tilde{\mathcal{O}}\!\left(\max\{P_T^{1/3}T^{2/3}\!\!,\; \sqrt{T}\}\right).
  \tag*{(\Cref{thm:dynamic-dyadic-master})}
\end{align*}
Combining the two yields the unified guarantee
\begin{align}
    \label{eq:result-two-branches}
  \E\bigl[\Rind_T\bigr]
  \;\le\;
  \tilde{\mathcal{O}}\!\left(
    \min\Bigl\{
      \sqrt{(S_T+1)\,T},
      \;\;
      (P_T+1)^{1/3}\,T^{2/3} 
    \Bigr\}
  \right).
\end{align}

\textbf{Limitation.}
In norm-based SOCO, the bound ${\mathcal{O}}(\sqrt{T(1+P_T)})$ is attainable on the larger class $\cU_P$, and since $\cU_S(s) \subseteq \cU_P(Ds)$, that single bound automatically yields the optimal ${\mathcal{O}}(\sqrt{T(1+S_T)})$ on $\cU_S$. Under indicator switching costs we do not preserve this unification: we recover $\tilde{\mathcal{O}}(\sqrt{T(S_T+1)})$ on $\cU_S$, but on $\cU_P$ our bound degrades to $\tilde{\mathcal{O}}(T^{2/3} P_T^{1/3})$. Whether \(\tilde{\mathcal{O}}(\sqrt{T(1+P_T)})\) is achievable under indicator switching remains open. The difficulty is that the learner is charged per change, whereas the comparator is charged only by total movement $P_T$.

\subsection{Overview of the approach}
To streamline presentation, we sketch here the algorithmic structure and analysis strategy. The algorithm has two layers: a base layer of restarted lazy randomized FTRL learners, one per dyadic scale, and a meta layer that aggregates them based on their expected loss and decision change.

\textbf{Base layer.}
For each dyadic scale $H \in \mathcal H \doteq \{1, 2, 4, \dots, 2^{\lfloor \log_2 T \rfloor}\}$, we run a fresh copy of FPRLL that restarts at the beginning of every block of length $H$. Each copy maintains a probability density $\mathcal Q_t^{(H)}$ over $\mathcal X$, induced by a perturbed FTRL objective. Small scales adapt quickly but restart, and hence switch, often; large scales are stable but slow to respond to non-stationarity. The base layer thus spans the full tradeoff between responsiveness and switching control.
\textbf{Meta layer:}
At each round $t$, the master maintains weights $\vec v_t \in \Delta_{\mathcal{H}}$
and forms the mixture density
\begin{align}
    \label{eq:master-mixture}
    {\mathcal{P}}_t \;\doteq\; \sum_{H\in\mathcal{H}} v_{t,H}\,\mathcal{Q}_t^{(H)}.
\end{align}
Action $\vec{x}_t$ is sampled from ${\mathcal{P}}_t$ via maximal coupling with ${\mathcal{P}}_{t-1}$, so that $\Pr(\vec{x}_t \neq \vec{x}_{t-1}) \!=\! \|{\mathcal{P}}_t\! -\! {\mathcal{P}}_{t-1}\|_{\mathrm{TV}}$, where $\|\cdot\|_{\mathrm{TV}}$ is total variation distance. After observing $f_t$, the master scores each scale $H$ by the surrogate loss $g_t^{(H)} \!\doteq\! \mathbb{E}_{\vec{x}\sim\mathcal{Q}_t^{(H)}}[f_t(\vec{x})]
\!+\! \lambda\,\|\mathcal{Q}_t^{(H)}\!-\!\mathcal{Q}_{t-1}^{(H)}\|_{\mathrm{TV}},$ which charges each scale for both its expected loss and the switching cost it would have incurred had the learner followed it alone.

We learn the mixing weights using the discounted-normal predictor of \citep{daniely2019competitive}. It achieves strongly adaptive regret, hence competing with every sub-interval, while also controlling the number of transitions between experts within it. In our setting, this ensures the master performs nearly as the best  restart scale for that interval, without paying an uncontrolled cost for switching between scales.

Analytically, at the base layer, we fix a dyadic scale $H$ and characterize the dynamic regret of FPRLL restarted every $H$ rounds
(\Cref{thm:dyadic-restart}). The resulting bound exposes the tradeoff
between the cost of restarts and the cost of comparator variation. 
At the meta layer, a total-variation decomposition of the combined actions, and the surrogate loss design, enable a reduction to the ``$H$-experts $\lambda/2$-switching'' problem on $\mathcal H$ (\cref{lem:one-step-master-dyadic}), which is handled via the discounted-normal predictor to obtain a strongly adaptive guarantee at the meta level.
Combining the two layers yields our two branches in \eqref{eq:result-two-branches}: for $\mathcal U_S$, we partition the horizon into piecewise-constant segments, cover each by its dyadic decomposition, and sum the base regret
against the segment's static comparator, giving
$\tilde{\mathcal O}(\sqrt{(S_T+1)T})$. For $ \mathcal U_P$,
sparsity is uninformative, but some dyadic scale in $\mathcal H$ always
matches the local rate of comparator drift; the master
tracks it, yielding $\tilde{\mathcal O}(T^{2/3} P_T^{1/3})$.

\section{The FPRLL base learner}
\label{sec:FPRLL}
\textbf{Setup.} Let $\mathcal{X} \doteq \{\vec{x}\in\mathbb{R}^d : s_c(\vec{x})\ge 0,\ c=1,\dots,C\}$
with each $s_c$ concave (so $\mathcal{X}$ is convex). Without loss of generality, $\vec{0}$ is strictly feasible
and $s_c(\vec{0})=1$. For $\gamma\in(0,1)$, we define the shrunk domain
$\mathcal{X}^\gamma \doteq (1-\gamma)\mathcal{X}$ and the scaled log-barrier:
\[
b_\gamma(\vec{x}) \doteq \frac{b(\vec{x})}{M_\gamma},
\qquad
b(\vec{x}) \doteq \sum_{c=1}^C \log\!\frac{\bar s_c}{s_c(\vec{x})},
\qquad
\bar s_c \doteq \sup_{\vec z\in\mathcal X} s_c(\vec z),
\qquad
M_\gamma \doteq \max\Bigl\{1,\,\sup_{\vec z\in\mathcal X^\gamma} b(\vec z)\Bigr\}.
\]
This is the standard logarithmic barrier, shifted to make \(b\ge0\) on \(\mathcal X\) and scaled so that \(b_\gamma\le1\) on \(\mathcal X^\gamma\),
as in \cite{shermanArxiv}. We let $G_c\!\doteq\!\sup_{\vec x \in \mathcal X^\gamma}\!\|\nabla s_c(\vec x)\|, \forall c$. Next, we define the $\sigma$-strongly convex function $F_t$:
\begin{align}
	\label{eq:f0-Ft}
	F_t \doteq \sum_{\tau=0}^t f_\tau, \quad \text{with} \quad	f_0 \doteq r + b_\gamma,
	\quad r(\vec{x})\doteq\tfrac{\sigma}{2}\|\vec{x}\|^2.
\end{align}

\textbf{Minimizers and their densities.}
Let $\vec{p}\sim\mathrm{Lap}(\mu)$ have density
$\nu(\vec{p}) = (2\mu)^{-d}\exp(-\|\vec{p}\|_1/\mu)$.
The \emph{fresh-perturbation minimizer} at round $t$ is
\begin{align}
\label{eq:fresh-minimizer}
\tilde{\vec{x}}_{t+1}
\doteq
\argmin_{\vec{x}\in\mathcal{X}}\bigl\{F_t(\vec{x})+\langle\vec{p}_t,\vec{x}\rangle\bigr\},
\end{align}
with induced distribution $\mathcal{Q}_{t+1}$. Strong convexity makes
$\vec{x}\mapsto -\nabla F_t(\vec{x})$ a smooth bijection on
$\mathrm{int}(\mathcal{X})$, so the change-of-variables formula
(\cref{lem:fresh-minimizer-density}) gives
\begin{align}
\label{eq:density}
\mathcal{Q}_t(\vec{x})
=
\nu\bigl(-\nabla F_{t-1}(\vec{x})\bigr)\,
\bigl|\det\bigl(-\nabla^2 F_{t-1}(\vec{x})\bigr)\bigr|,
\qquad \vec{x}\in\mathrm{int}(\mathcal{X}).
\end{align}
The closed form lets us evaluate $\mathcal{Q}_t$ \emph{pointwise}, which the lazy coupling below exploits directly.

\textbf{Two perspectives.}
\cref{main-alg} can be deployed in two ways. As a standalone algorithm,
it lazy-samples an action $\vec{x}_t\sim\mathcal{Q}_t$ at each round
(lines 5--12), \emph{maximally coupling} consecutive draws so that $\Pr(\vec{x}_{t+1}\neq \vec{x}_t)
\;=\;
\|\mathcal{Q}_t - \mathcal{Q}_{t+1}\|_{\mathrm{TV}}.$
Alternatively, the meta-algorithm of \cref{sec:meta} consumes
\cref{main-alg} as a \emph{density oracle}: it queries $\mathcal{Q}_t$
without invoking the sampler, and lazy-samples instead at the
mixture level. Either way, the analysis below targets the same
per-round quantity,
\begin{align*}
\mathbb{E}_{\vec{x}\sim\mathcal{Q}_t}[f_t(\vec{x})]
\;+\;
\|\mathcal{Q}_t-\mathcal{Q}_{t-1}\|_{\mathrm{TV}},
\end{align*}
which is the expected regret-plus-switching cost
when the algorithm is run standalone, and also  the surrogate score charged to each base learner by the meta-algorithm.

\begin{algorithm}[t]
\caption{Follow the Perturbed Regularized Lazy Leader (FPRLL)}
\label{main-alg}
\begin{algorithmic}[1]
\Require horizon $T$, domain $\mathcal{X}$, parameters $\gamma,\sigma,\mu>0$
\Ensure actions $\vec{x}_1,\dots,\vec{x}_T$
\State Sample $\vec{p}_0\sim\mathrm{Lap}(\mu)$;\;
       set $\vec{x}_1\gets\argmin_{\vec{x}\in\mathcal{X}}\{f_0(\vec{x})+\langle\vec{p}_0,\vec{x}\rangle\}$
\For{$t=1,\dots,T$}
    \State Play $\vec{x}_t$;\; observe $f_t$ and incur loss $f_t(\vec{x}_t)$
    \State Evaluate $\mathcal{Q}_t(\vec{x}_t)$ and $\mathcal{Q}_{t+1}(\vec{x}_t)$ via \eqref{eq:density};\;
           sample $z\sim\mathrm{Unif}[0,\,\mathcal{Q}_t(\vec{x}_t)]$
    \If{$z<\mathcal{Q}_{t+1}(\vec{x}_t)$}
        \State $\vec{x}_{t+1}\gets\vec{x}_t$ \hfill\Comment{retain: switching probability $=0$ here}
    \Else
        \Repeat \hfill\Comment{sample from residual of $\mathcal{Q}_{t+1}$}
            \State Sample $\vec{p}\sim\mathrm{Lap}(\mu)$;\;
                   solve $\vec{y}\gets\argmin_{\vec{x}\in\mathcal{X}}\{F_t(\vec{x})+\langle\vec{p},\vec{x}\rangle\}$
            \State Evaluate $\mathcal{Q}_{t+1}(\vec{y})$ and $\mathcal{Q}_t(\vec{y})$ via \eqref{eq:density};\;
                   sample $z'\sim\mathrm{Unif}[0,\,\mathcal{Q}_{t+1}(\vec{y})]$
        \Until{$z'>\mathcal{Q}_t(\vec{y})$}
        \State $\vec{x}_{t+1}\gets\vec{y}$
    \EndIf
\EndFor
\end{algorithmic}
\end{algorithm}

\textbf{Regret guarantees.}
The following statements characterize the performance of FPRLL and its restarted variants.
The first statement describes  dynamic-regret of
FPRLL, explicit in all tunable parameters $(\sigma,\mu)$ and
the comparator path budget $P_T$.
\begin{theorem}[Dynamic regret of FPRLL]
\label{thm:agnostic-dynamic}
If \(\mathrm{FPRLL}\) is run with $\gamma=\frac{1}{\sqrt T}$,
then for any \(\sigma,\mu>0\),
\[
\expec{}{\Rind_T}
\le
\frac{G^2}{2\sigma}T
+R\sqrt{2d}\,\mu
+(TG+\sigma R)P_T
+\bigl(CG_cP_T+GR\bigr)\sqrt T
+\lambda T\Big(\frac{\beta d}{\sigma}{+}\frac{\sqrt d\,G}{\mu}
\Big)
+\frac{\sigma}{2}R^2
+1.
\]
In particular, choosing
$
\mu^\star=\sqrt{\frac{\lambda G T}{R\sqrt 2}},
\qquad
\sigma^\star=\sqrt{\frac{(G^2+2\lambda\beta d)\,T}{R^2}},
$
yields
\[
\expec{} {\Rind_T}\le
(R+P_T)\sqrt{(G^2+2\lambda\beta d)\,T}
+3\sqrt{\lambda dRGT}
+TG\,P_T
+\bigl(CG_cP_T+GR\bigr)\sqrt T
+1.
\]
\end{theorem}

Specializing the above result to a \emph{fixed} comparator on a length-$k$ block
and optimizing $(\gamma,\sigma,\mu)$ accordingly yields an
$\mathcal{O}(C_{\mathrm{base}}\sqrt{k})$ static-regret bound with indicator switching cost,
which is a per-block primitive used throughout the multiscale analysis.

\begin{corollary}[Static regret on a block]
\label{cor:static-block}
Consider \(\mathrm{FPRLL}\), and let
$I=\{(j-1)k+1,\dots,jk\}$
be any full restart block. On this block, the algorithm is restarted and run with
$\gamma=\frac{1}{\sqrt{k}},
\mu^\star_k=\sqrt{\frac{\lambda Gk}{R\sqrt 2}},
\sigma^\star_k=\frac{\sqrt{(G^2+2\lambda\beta d)\,k}}{R}.$
Then, for every fixed comparator \(\vec u\in\X\),
\begin{align*}
    \expec{}{\Rind_I(\vec u)} &\doteq 
    \mathbb E\Big[
    \sum_{t\in I} \bigl(f_t(\vec x_t)-f_t(\vec u)\bigr)
        +\lambda\sum_{t=(j-1)k+1}^{jk} \mathbf 1\{x_t\neq x_{t-1}\}
    \Big]
    \\
    &\le
    R\sqrt{(G^2+2\lambda\beta d)\,k}
    +3\sqrt{\lambda dRG\,k}
    +GR\sqrt{k}
    +\lambda+1 = C_{\mathrm{base}}\sqrt{k},
\end{align*}
where $C_{\mathrm{base}}
\doteq
R\sqrt{G^2+2\lambda\beta d}
+3\sqrt{\lambda dRG}
+GR+\lambda+1.$
\end{corollary}

For each $k$, let FPRLL$(k)$ denote the algorithm that restarts FPRLL every $k$ rounds, and let $\RindH{k}\doteq\expec{}{\Rind_T(\comseq)}$ be its expected dynamic regret against $\comseq$. The following theorem shows that, among the dyadic restart periods
$k\in\mathcal{H}$, there always exists one that achieves a balance between laziness cost and tracking cost.

\begin{theorem}[A dyadic restart scale]
\label{thm:dyadic-restart}
Assume for simplicity that \(T\) is a power of two, and let $\mathcal H \doteq \{1,2,4,\dots,T\}$.
For each $k\in\mathcal H$, FPRLL$(k)$ uses within each block the tuning
$\gamma=\frac{1}{\sqrt{k}},\
\mu^\star_k=\sqrt{\frac{\lambda Gk}{R\sqrt 2}},\
\sigma^\star_k=\frac{\sqrt{(G^2+2\lambda\beta d)\,k}}{R}.$
Define the constants
$A \doteq
C_{\mathrm{base}}+\lambda,
\
B \doteq
\sqrt{G^2+2\lambda\beta d}+CG_c.$
Then there exists a dyadic restart scale $k^\dagger\in\mathcal H$ such that
\[
\RindH{k^\dagger}_T
\;\le\;
\sqrt{2}\,
\min_{1\le k\le T}
\left(
A\frac{T}{\sqrt{k}}
+
(B+G)P_T\,k
\right).
\]
In particular, for $P_T=\bigo(T^\alpha)$, $\alpha\in[0,1)$,
$\RindH{k^\dagger} = \bigo\big(T^{2/3}P_T^{1/3}\big).$
\end{theorem}

We complement \cref{thm:dyadic-restart} with a matching lower bound, up to
logarithmic factors, for restarted randomized lazy learners of the form in
\cref{main-alg}. The lower bound exposes a tradeoff in the restart period
$k$ between tracking ($f_t(\vec x_t) - f_t(\vec u_t)$), and laziness costs $\ind$.

\textbf{Tracking cost.} A long restart period prevents the learner from reacting to comparator
switches that occur inside a block. The next theorem captures this through a
single oblivious distribution that is hard simultaneously for all $k$.

\begin{theorem}[Tracking lower bound]
\label{thm:dyadic-restart-lb}
Consider online linear optimization on $\mathcal{X}=[-1,1]$ with losses
$f_t(x)=g_t x$, $g_t\in\{-1,+1\}$. Let $\tau\ge 1$ and $T\ge 4\tau$.
There exists a distribution $\mathcal{D}$ over oblivious loss sequences such
that, for every integer $k\in[4,\,T/\tau]$, FPRLL$(k)$ run against
$\mathcal{D}$ satisfies
\[
\E\bigl[\Rind_T(u_{1:T})\bigr]
\;=\;
\Omega\left(\frac{k\tau}{\log(T/\tau)}\right)
\]
for the associated comparator sequence with $S_T(u_{1:T})=\tau$ and
$P_T(u_{1:T})=2\tau$. 
\end{theorem}

Thus the tracking term is $\Omega(k\tau/\log(T/\tau))$, which worsens with
$k$.

\textbf{Laziness cost.}
A short restart period forces repeated starts from scratch.
Adapting \citep[Thm.~4]{shermanArxiv}, on a length-$k$ block, each algorithm
that makes $O(\sqrt{k})$ expected switches must incur
$\Omega(\sqrt{k})$ expected regret against a fixed comparator.
Replaying this hard
instance over the $T/k$ blocks gives $\Omega(T/\sqrt{k})$ regret
against a stationary comparator with $S_T\!=\!P_T\!=\!0$; see
\cref{lem:sherman-branch}. This term improves with $k$.

For every fixed $k$, mixing the tracking adversary with the
laziness adversary tuned to that scale gives
\begin{equation}
\E\bigl[\Rind_T\bigr]
\;=\;
\Omega\left(
\frac{k\tau}{\log(T/\tau)}
+
\frac{T}{\sqrt{k}}
\right)
\tag*{(\cref{lem:combined-lb})}.
\end{equation}

\section{Meta-learning framework}
\label{sec:meta}
We now construct a meta-learner on top of a number of restarted base
experts. First, we invoke a geometric cover fact that controls how
intervals decompose across dyadic scales. Second, we model each scale
$H\in\mathcal H$ as an expert that restarts every $H$ rounds. Third, we
attach to each expert a surrogate loss that combines its expected loss
with its switching cost, aggregate the experts into a mixture
$\mathcal P_t$, and sample via lazy coupling of the mixtures, reducing
the master to online linear optimization on the simplex with an $\ell_1$
switching penalty. Finally, we solve this reduced problem with the
strongly adaptive algorithm of \citep{daniely2019competitive}.

\textbf{Dyadic schedule and a geometric cover.}
For
$j=0,1,\dots,\lfloor \log_2 T\rfloor$, define the family of level-$j$
dyadic intervals
$
\mathcal G_j
\doteq
\bigl\{[(i-1)2^j+1,\, i2^j] : i\in\mathbb N,\ i2^j\le T\bigr\}$, and collect
\[
\mathcal G \doteq \bigcup_{j=0}^{\lfloor \log_2 T\rfloor}\mathcal G_j,
\qquad
\mathcal H \doteq \{2^j : 0\le j\le \lfloor \log_2 T\rfloor\},
\qquad
K \doteq |\mathcal H| = \lfloor \log_2 T\rfloor+1.
\]
Thus $\mathcal G$ is the set of all dyadic intervals contained in $[T]$, and
$\mathcal H$ is the set of dyadic \emph{lengths}. The following standard
combinatorial fact shows that every
interval in $[T]$ admits a short dyadic cover whose lengths sum to a term of the order $\sqrt{L}$. 

\begin{lemma}[Dyadic geometric cover]
\label{lem:dyadic-cover}
For every interval $I=[s,e]\subseteq[T]$ of length $L=e-s+1$, there exist
consecutive dyadic intervals $J_1,\dots,J_m\in\mathcal G$ partitioning $I$
such that
\begin{equation}
\label{eq:dyadic-cover}
m \le 2\lceil \log_2 L\rceil + 2
\qquad\text{and}\qquad
\sum_{r=1}^m \sqrt{|J_r|} \le C_{\mathrm{gc}}\sqrt{L},
\qquad
C_{\mathrm{gc}} \doteq 2+\sqrt{2}.
\end{equation}
\end{lemma}
\textbf{Restarted FPRLL base at each dyadic scale.}
For each dyadic scale $H\in\mathcal H$, we define a base expert
$\mathcal E_H$ that runs a fresh FPRLL$(H)$ instance on every
consecutive block of length $H$, restarting at the beginning of each new
block (\cref{cor:static-block}). Let $\mathcal Q_t^{(H)}$ denote the
marginal law of $\mathcal E_H$'s action at round $t$. We attach to each
expert its instantaneous expected loss, TV movement, and surrogate loss:
\begin{equation}
\label{eq:expert-losses}
\ell_t^{(H)} \doteq \E_{\vec x\sim \mathcal Q_t^{(H)}}[f_t(\vec x)],
\qquad
c_t^{(H)} \doteq \|\mathcal Q_t^{(H)}-\mathcal Q_{t-1}^{(H)}\|_{\mathrm{TV}},
\qquad
g_t^{(H)} \doteq \ell_t^{(H)} + \lambda\, c_t^{(H)},
\end{equation}
with the convention $c_1^{(H)}\doteq 0$. We state the main analysis using the exact quantities
\(\ell_t^{(H)}\) and \(c_t^{(H)}\) for clarity. In practice, these
quantities need not be available in closed form; a practical implementation can
therefore feed the meta-learner a sampled surrogate whose conditional
mean is an easily computable upper bound on \(g_t^{(H)}\). Details are
deferred to \cref{subsec:hat-surrogate}.

The surrogate $g_t^{(H)}$ is the natural expert loss for our problem: it
already prices in the switching cost that the expert would incur if the
master followed it exclusively. Indeed, if actions were sampled directly
from $\mathcal Q_t^{(H)}$ under the lazy coupling of
\cref{lem:lazy-sampling}, then $\Pr(\vec x_t\neq \vec x_{t-1})=c_t^{(H)}$,
and therefore for every comparator $\comseq\in\X^T$,
\begin{equation}
\label{eq:expert-regret-as-surrogate}
\E\left[\sum_{t=1}^T f_t(\vec x_t)+\lambda\sum_{t=2}^T
\mathbf 1\{\vec x_t\neq \vec x_{t-1}\}\right]
-\sum_{t=1}^T f_t(\vec u_t)
=
\sum_{t=1}^T \bigl(g_t^{(H)}-f_t(\vec u_t)\bigr) = \RindH{H}_T.
\end{equation}
Hence a meta-learner that competes with the best cumulative surrogate
automatically inherits both the prediction \emph{and} the switching
guarantees of the corresponding expert.

\textbf{Density-mixture master and the reduction.}
At round $t$, the meta-learner plays a weight vector
$
\vec v_t \in \Delta_K
\doteq
\Bigl\{\vec v\in\mathbb R_+^{\mathcal H}
: \textstyle\sum_{H\in\mathcal H} v_H = 1\Bigr\},$
and forms the mixture marginal as in \cref{eq:master-mixture}
The played action $\vec x_t$ is then drawn from $\mathcal P_t$ via the lazy
coupling with $\mathcal P_{t-1}$, which guarantees
$\vec x_t\sim\mathcal P_t$ and
$\Pr(\vec x_t\neq \vec x_{t-1})=\|\mathcal P_t-\mathcal P_{t-1}\|_{\mathrm{TV}}$.
The full procedure is summarized in \cref{alg:master} below.

\begin{algorithm}[h]
\caption{Density-mixture master (instantiation of the reduction)}
\label{alg:master}
\begin{algorithmic}[1]
\Require base experts $\{\mathcal E_H\}_{H\in\mathcal H}$; meta-learner
         $\mathcal M$ satisfying \cref{lem:dm-imported} with
         $N=K$ experts and switching cost $D=\lambda/M$
\Ensure actions $\vec x_1,\dots,\vec x_T$
\State Sample $\vec x_1\sim \mathcal P_1$
\For{$t=1,2,\dots,T$}
  \State Play $\vec x_t$, observe $f_t$
  \State For each $H\in\mathcal H$: update $\mathcal E_H$ to obtain
         $\mathcal Q_{t+1}^{(H)}$
         \Comment{restarted FPRLL$(H)$}
  \State Form surrogate losses
         $g_t^{(H)} = \E_{\mathcal Q_t^{(H)}}[f_t]
                      + \lambda\,\|\mathcal Q_t^{(H)}-\mathcal Q_{t-1}^{(H)}\|_{\mathrm{TV}}$
  \State Feed $\vec\ell_t \doteq \vec g_t / M \in [0,1]^K$ to $\mathcal M$;
         receive $\vec v_{t+1}\in\Delta_K$
         \Comment{\cref{alg:dm-multiscale}}
  \State Form mixture
         $\mathcal P_{t+1}\gets\sum_{H\in\mathcal H} v_{t+1,H}\,\mathcal Q_{t+1}^{(H)}$
  \State $\vec x_{t+1}\gets \textsc{LazySample}(\vec x_t,\mathcal P_t,
         \mathcal P_{t+1})$
         \Comment{\cref{alg:lazy-sample}}
\EndFor
\end{algorithmic}
\end{algorithm}

\cref{alg:master} is also efficient. The meta-learner of
\citet{daniely2019competitive} has per-round complexity \(\mathcal{O}(K\log T)\)
for \(K\) experts, and in our construction \(K=|\mathcal H|=\mathcal{O}(\log T)\).
Thus the meta-level overhead is \(\mathcal{O}(\log^2 T)\) per round. The remaining
cost is that of updating the \(\mathcal{O}(\log T)\) base learners, each a copy of
FPRLL. Although \cref{alg:master} writes the exact surrogate
losses for ease of presentation, in implementation we feed the
meta-learner the sampled surrogate losses described in
\cref{subsec:hat-surrogate}. Thus the per-round cost is the cost of
\(\mathcal{O}(\log T)\) FPRLL updates, plus an additional
\(\mathcal{O}(\log^2 T)\) meta-level overhead and the cost of computing the
implementable surrogate losses.

The next lemma is the structural backbone of the analysis: it decomposes
the switching probability of the master into an \emph{expert-movement} term
(the average TV movement of the dyadic coordinates) and a
\emph{weight-movement} term (the $\ell_1$ movement of the meta weights).

\begin{lemma}[One-step master reduction]
\label{lem:one-step-master-dyadic}
Let $\vec g_t \doteq (g_t^{(H)})_{H\in\mathcal H}$. For every $t\ge 2$,
\begin{equation}
\label{eq:mixture-tv-dyadic}
\|{\mathcal P}_t-{\mathcal P}_{t-1}\|_{\mathrm{TV}}
\;\le\;
\sum_{H\in\mathcal H} v_{t,H}\, c_t^{(H)}
+
\tfrac12\|\vec v_t-\vec v_{t-1}\|_1,
\end{equation}
\vspace{-4mm}
\begin{align}
\label{eq:one-step-master-dyadic}
\text{and consequently} \qquad \E[f_t(\vec x_t)] + \lambda \Pr(\vec x_t\neq \vec x_{t-1})
\;\le\;
\langle \vec v_t, \vec g_t\rangle
+
\tfrac{\lambda}{2}\|\vec v_t-\vec v_{t-1}\|_1.
\end{align}
\end{lemma}

The interpretation of \cref{eq:one-step-master-dyadic} is immediate: the
meta-learner sees expert losses $\vec g_t$ that already price in each
expert's own TV movement, and pays only for its own weight movement in
$\ell_1$. We have therefore reduced the master problem to an
\emph{experts-with-switching-costs} problem over the $K=\mathcal{O}(\log T)$ dyadic
scales, with switching-cost parameter $\lambda/2$.

\textbf{Strongly adaptive guarantee for the meta-learner.}
To realize the reduction, we feed the surrogate losses $\{\vec g_t\}$ into
the strongly adaptive algorithm of \citep{daniely2019competitive}, which competes with the best expert on \emph{every} interval while controlling switches among experts. Its guarantees require a bounded surrogate range, which is ensured by \cref{assump:bounded-function} and \cref{assump:switching}. Hence, $g_t^{(H)} \le M_f +\lambda \doteq M$.

\begin{lemma}[Meta-regret]

\label{lem:meta-dyadic}
There exists an online algorithm producing weights $\vec v_t\!\in\!\Delta_K$ such that for every interval $[s,e]\subseteq[T]$ of length $L\!=\!e\!-s\!+1$, there is a universal constant $C_{\mathrm{dm}}>0$ for which:
	\begin{equation}
		\label{eq:meta-dyadic}
		\sum_{t=s}^{e}\langle \vec v_t,\vec g_t\rangle
		+
		\tfrac{\lambda}{2}\sum_{t=s+1}^{e}\|\vec v_t-\vec v_{t-1}\|_1
		\;\le\;
		\min_{H\in\mathcal H}\sum_{t=s}^{e} g_t^{(H)}
		+ C_{\mathrm{dm}}\sqrt{M(M+\lambda)\,L\log(KT)},
	\end{equation}
\end{lemma}

\section{Regret guarantees}
\label{sec:regret}

We now combine the base-learner analysis of \cref{sec:FPRLL} with the meta-learner of \cref{sec:meta} to derive our main regret bounds. The argument exploits the meta-learner's strongly adaptive property in two distinct ways. First, pairing it with the \emph{static}-regret guarantee of the base learner on each restart block yields a strongly adaptive regret (SAR) bound, which in turn implies a dynamic-regret bound against piecewise-constant comparators. Second, pairing it with the \emph{dynamic}-regret guarantee of the base learner at the best dyadic restart scale yields a path-length-based bound against arbitrary comparators. 

\subsection{Piecewise-constant comparators via strongly adaptive regret}
\label{subsec:pwc}

The interval regret with indicator switching penalty on $I=[s,e]\subseteq[T]$ is
\begin{align}
    \label{eq:sar-def}
    \mathrm{SAR}_I(\vec u) \doteq \E\left[\sum_{t=s}^{e} f_t(\vec x_t) + \lambda\sum_{t=s+1}^{e}\one\{\vec x_t \neq \vec x_{t-1}\}\right] - \sum_{t=s}^{e} f_t(\vec u).
\end{align}

\cref{alg:master} attains $\mathrm{SAR}_I(\vec u) = \tilde {\mathcal{O}}(\sqrt L)$  \emph{simultaneously} on all intervals of length $L$.
\begin{theorem}[Strongly adaptive regret]
\label{thm:reduction-geometric}
For every interval $I=[s,e]\subseteq[T]$ of length $L=e-s+1$,
\begin{equation}
\label{eq:sar-bound}
\mathrm{SAR}_I(\vec u)
\le
C_{\mathrm{base}}C_{\mathrm{gc}}\sqrt{L}
+
C_{\mathrm{dm}}\,C_{\mathrm{gc}}\sqrt{M(M+\lambda)L\log(KT)}
+
\lambda\bigl(2\lceil\log_2 L\rceil+1\bigr)
= \tilde {\mathcal{O}}(\sqrt{L}).
\end{equation}
\end{theorem}

\begin{proof}[Proof sketch]
The bound combines three ingredients already established: \textbf{(i)} \emph{Dyadic cover}: \cref{lem:dyadic-cover} partitions $I$ into $m = \mathcal{O}(\log L)$ consecutive dyadic blocks $J_1,\dots,J_m\in\cG$ with $\sum_{r=1}^m \sqrt{|J_r|}\le C_{\mathrm{gc}}\sqrt{L}$. \textbf{(ii)} \emph{Meta-level interval regret}: applied blockwise, \cref{lem:meta-dyadic} ensures that on each $J_r$ the meta-learner competes with the unique dyadic expert whose restart period equals $|J_r|$, at an $\tilde {\mathcal{O}}(\sqrt{|J_r|})$ cost. \textbf{(iii)} \emph{Base-level static regret}: on its matched restart block, that expert is by definition a fresh run of FPRLL on a horizon of length $|J_r|$, which by \cref{cor:static-block} has static regret at most $C_{\mathrm{base}}\sqrt{|J_r|}$.

The one-step bound of \cref{lem:one-step-master-dyadic} lifts the meta inner-product bound to the master cost, and boundary movement of $\vec v_t$ across blocks is controlled by the simplex bound $\|\vec v_t-\vec v_{t-1}\|_1\le 2$. Summing (i)--(iii) yields \cref{eq:sar-bound}; the full argument is given in \cref{app:sar-proof}.
\end{proof}

A SAR-to-dynamic-regret reduction based on segmenting the comparator at its switch points (paying a non-dominant $\lambda S_T$ term) yields the piecewise-constant branch of our guarantee (see \cref{app:sar-to-dr-pwc}):
\begin{corollary}[Dynamic regret, $S_T$ branch]
\label{cor:dr-pwc}
For every comparator sequence, \cref{alg:master} attains
\[
\E\bigl[\Rind_T(\comseq)\bigr]
\;\le\;
\tilde {\mathcal{O}}\left(\sqrt{(S_T+1)\,T}\right) + \lambda\,S_T.
\]
\end{corollary}

\subsection{Path-varying comparators}
\label{subsec:pt}

For comparators with bounded path length, segmenting at switch points is no longer informative since $\vec u_t$ may drift by arbitrarily small amounts at every round. Instead, we invoke the base learner's dynamic-regret characterization (\cref{thm:dyadic-restart}), which identifies the optimal dyadic restart scale $k^\dagger\in\cH$ matching the comparator's drift. For the meta-learner notation, this is the expert with scale \(H=k^\dagger\).

\begin{theorem}[Dynamic regret, $P_T$ branch]
\label{thm:dynamic-dyadic-master}
For every comparator sequence, \cref{alg:master} attains
\begin{equation}
\label{eq:dr-meta}
\E\bigl[\Rind_T(\comseq)\bigr]
\;\le\;
\min_{H\in\cH}\left(\sum_{t=1}^T g_t^{(H)} - \sum_{t=1}^T f_t(\vec u_t)\right)
\;+\; C_{\mathrm{dm}}\sqrt{M(M+\lambda)\,T\log(KT)}.
\end{equation}
In particular, for $H = k^\dagger$ from  \cref{thm:dyadic-restart}, we obtain 
\begin{align}
\label{eq:dr-final}
&\E\bigl[\Rind_T(\comseq)\bigr]\;=\; \tilde {\mathcal{O}}
\left(\max\left\{T^{2/3}P_T^{1/3} ,\, \sqrt T\right\}\right).
\end{align}
\end{theorem}

\begin{proof}[Proof sketch]
The one-step bound (\cref{lem:one-step-master-dyadic}), summed over $t=1,\dots,T$, reduces the master cost to the meta inner-product loss plus a TV movement term. Applying \cref{lem:meta-dyadic} on the full interval $[1,T]$ yields \cref{eq:dr-meta}. Substituting $H=k^\dagger$ and invoking the base-level dynamic regret bound of \cref{thm:dyadic-restart} gives \cref{eq:dr-final}. See \cref{app:dr-proof} for details.
\end{proof}

\section{Related work}
Our work connects two largely separate lines of research: algorithms with few learner decision \emph{changes}, and algorithms with guarantees against \emph{moving} comparators. Existing results typically provide one of these guarantees, but not both, for general OCO.

\textbf{Lazy OCO.}
A line of work studies online learners constrained to switch $o(T)$ times against a \emph{static} comparator. \citep{anava2015online} obtained $\sqrt{dT}+dT/S$ regret with $S$ expected switches. 
\citep{jaghargh2019consistent} studied this setting under the name ``consistent'' online optimization and used Poisson Process  schedules, but obtained a weaker tradeoff in terms of $T$.
\citep{sherman2021lazy} and \citep{agarwal2023differentially} improved the \citep{anava2015online} guarantee to $\sqrt{T}+dT/S$, and \citep{agarwal2024improved} sharpened this to $\sqrt{T}+\sqrt{d}T/S$ while removing smoothness assumptions via log-concave sampling. \citep{wang2021online} obtained the same $T/S$-type dependence for $S$ switches deterministically, but under a continuous norm-based switching constraint. These works typically present regret and switching count separately; we report their sum, which is equivalent via standard reductions \citep[Lem.~17]{sherman2021lazy} and directly captures the optimal joint tradeoff. For the \emph{strongly} convex case, \citep{sherman2021lazy} obtains better bounds of order $\widetilde{\bigo}(T/S^2)$. Against \emph{adaptive} adversaries, the problem is provably harder: \citep{chen2020minimax} established a $\Theta(T/\sqrt{S})$ minimax rate for an expected $S$ switches. All of these results target a static comparator across the time horizon.

\textbf{Moving comparators and switching costs.}
A second line studies dynamic regret under \emph{norm-based} movement penalties $\|\vec x_t-\vec x_{t-1}\|$ or other Lipschitz switching costs \citep{chen2018smoothed, zhao2020understand, zhang2021revisiting, mhaisen2026partially}. \citep{zhang2021revisiting} achieved the optimal $\bigo(\sqrt{T(1+P_T)})$ dynamic regret by incorporating switching cost into a meta-aggregation loss. Our approach is similar in spirit, but aggregates over total-variation distance and uses different meta- and base-learners. Most recently, \citep{mhaisen2026partially} showed that FTRL itself can attain dynamic-regret guarantees with additional laziness properties such as staleness. Relatedly, \citep{zhang2022optimal} studies unbounded comparators, and \citep{esposito2026parameter} extends this setting to moving unbounded comparators with bounds adaptive to individual switching-cost parameters $\lambda_t$. However, these costs remain norm-based, and the resulting algorithms may still change decisions every round. Strongly adaptive methods provide another route to dynamic or tracking regret: \citep{daniely2015strongly} introduced geometric-covering reductions from static regret to strongly adaptive regret, and \citep{zhang2018dynamic} showed that strongly adaptive regret implies dynamic-regret bounds in OCO, typically with an additional function-variation term. \citep{zhang2022smoothed} extends this perspective to OCO with switching costs, but again under norm-based costs. 

Beyond the two lines above, our framework draws on adjacent threads.
Differentially private online learning naturally incentivizes few
decision changes, since each update consumes privacy budget, and hence
algorithms developed in one setting often transfer to the other
\citep{agarwal2023differentially}. On the algorithmic side, our
master/base design follows the meta-learning-over-restarting-experts
line of Follow-the-Leading-History and the sleeping-experts perspective
\citep{adamskiy2012closer}. Restarting a randomized (perturbed) leader
to obtain dynamic and strongly adaptive regret was recently considered
by \citep{xu2024online}, but without switching costs of any kind
(deterministic, indicator, or normed) and without accompanying lower
bounds.

\bibliography{References.bib}

@article{orabona2021modern,
	title        = {{A Modern Introduction to Online Learning}},
	author       = {Orabona, Francesco},
	year         = 2022,
	publisher    = {arXiv},
	journal      = {arXiv.1912.13213},
}

@article{mcmahan-survey17,
	title        = {{A Survey of Algorithms and Analysis for Adaptive Online Learning}},
	author       = {McMahan, H. Brendan},
	year         = 2017,
	journal      = {J. Mach. Learn. Res.},
	volume       = 18,
	number       = 1,
	pages        = {3117--3166}
}

@inproceedings{zhang2018adaptive,
  title={Adaptive online learning in dynamic environments},
  author={Zhang, Lijun and Lu, Shiyin and Zhou, Zhi-Hua},
  booktitle={Proc. of NeurIPS},
  year={2018}
}

@inproceedings{zhang2018dynamic,
  title={Dynamic regret of strongly adaptive methods},
  author={Zhang, Lijun and Yang, Tianbao and Zhou, Zhi-Hua and others},
  booktitle={International conference on machine learning},
  pages={5882--5891},
  year={2018},
  organization={PMLR}
}

@inproceedings{jacobsen2022parameter,
	title        = {Parameter-free mirror descent},
	author       = {Jacobsen, Andrew and Cutkosky, Ashok},
	year         = 2022,
	booktitle    = {Proc. of COLT}
}

@inproceedings{zhang2021revisiting,
  title={Revisiting smoothed online learning},
  author={Zhang, Lijun and Jiang, Wei and Lu, Shiyin and Yang, Tianbao},
  booktitle={Proc. of NeurIPS},
  year={2021}
}

@article{zhao2020understand,
  title={Understand dynamic regret with switching cost for online decision making},
  author={Zhao, Yawei and Zhao, Qian and Zhang, Xingxing and Zhu, En and Liu, Xinwang and Yin, Jianping},
  journal={ACM Trans. on Intelligent Systems and Technology},
  volume={11},
  number={3},
  year={2020},
}

@inproceedings{cesa2013online,
  title={Online learning with switching costs and other adaptive adversaries},
  author={Cesa-Bianchi, Nicolo and Dekel, Ofer and Shamir, Ohad},
  booktitle={Proc. of NeurIPS},
  year={2013}
}

@inproceedings{chen2018smoothed,
  title={Smoothed online convex optimization in high dimensions via online balanced descent},
  author={Chen, Niangjun and Goel, Gautam and Wierman, Adam},
  booktitle={Proc. of COLT},
  year={2018},
}

@inproceedings{pruning-icml,
  title={On the Dynamic Regret of Following the Regularized Leader: Optimism with History Pruning},
  author={Mhaisen, Naram and Iosifidis, George},
  booktitle={Proc. of ICML},
  year={2025}
}

@inproceedings{zhang2022optimal,
  title={Optimal comparator adaptive online learning with switching cost},
  author={Zhang, Zhiyu and Cutkosky, Ashok and Paschalidis, Yannis},
  booktitle={Proc. of NeurIPS},
  year={2022}
}

@inproceedings{sherman2021lazy,
  title={Lazy oco: Online convex optimization on a switching budget},
  author={Sherman, Uri and Koren, Tomer},
  booktitle={Proc. of COLT},
  year={2021},
}

@article{shermanArxiv,
  title={Lazy OCO: Online Convex Optimization on a Switching Budget},
  author={Sherman, Uri and Koren, Tomer},
  journal={arXiv 2102.03803},
  year={2023}
}

@inproceedings{agarwal2023differentially,
  title={Differentially private and lazy online convex optimization},
  author={Agarwal, Naman and Kale, Satyen and Singh, Karan and Thakurta, Abhradeep},
  booktitle={Proc. of COLT},
  year={2023},
}

@inproceedings{zhang2022smoothed,
  title={Smoothed online convex optimization based on discounted-normal-predictor},
  author={Zhang, Lijun and Jiang, Wei and Yi, Jinfeng and Yang, Tianbao},
  booktitle={Proc. of NeurIPS},
  year={2022}
}

@inproceedings{anava2015online,
  title={Online learning for adversaries with memory: price of past mistakes},
  author={Anava, Oren and Hazan, Elad and Mannor, Shie},
  booktitle={Proc. of NeurIPS},
  year={2015}
}

@inproceedings{chen2020minimax,
  title={Minimax regret of switching-constrained online convex optimization: No phase transition},
  author={Chen, Lin and Yu, Qian and Lawrence, Hannah and Karbasi, Amin},
  booktitle={Proc. of NeurIPS},
  year={2020}
}

@inproceedings{mhaisen2026partially,
  title={Partially Lazy Gradient Descent for Smoothed Online Learning},
  author={Mhaisen, Naram and Iosifidis, George},
  booktitle={Proc. of AISTATS},
  year={2026}
}

@inproceedings{daniely2019competitive,
  title={Competitive ratio vs regret minimization: achieving the best of both worlds},
  author={Daniely, Amit and Mansour, Yishay},
  booktitle={Proc. of ALT},
  year={2019},
}

@inproceedings{agarwal2024improved,
  title={Improved differentially private and lazy online convex optimization: lower regret without smoothness requirements},
  author={Agarwal, Naman and Kale, Satyen and Singh, Karan and Thakurta, Abhradeep Guha},
  booktitle={Proc. of ICML},
  year={2024}
}

@inproceedings{daniely2015strongly,
  title={Strongly adaptive online learning},
  author={Daniely, Amit and Gonen, Alon and Shalev-Shwartz, Shai},
  booktitle={Proc. of ICML},
  year={2015},
}

@inproceedings{wang2021online,
  title={Online convex optimization with continuous switching constraint},
  author={Wang, Guanghui and Wan, Yuanyu and Yang, Tianbao and Zhang, Lijun},
  booktitle={Proc. of NeurIPS},
  year={2021}
}

@inproceedings{jaghargh2019consistent,
  title={Consistent online optimization: Convex and submodular},
  author={Jaghargh, Mohammad Reza Karimi and Krause, Andreas and Lattanzi, Silvio and Vassilvtiskii, Sergei},
  booktitle={Proc. of AISTATS},
  year={2019},
}

@inproceedings{esposito2026parameter,
  title={Parameter-free Dynamic Regret: Time-varying Movement Costs, Delayed Feedback, and Memory},
  author={Esposito, Emmanuel and Jacobsen, Andrew and Qiu, Hao and Zhang, Mengxiao},
  booktitle={Proc. of ICML},
  year={2026}
}

@inproceedings{altschuler2018online,
  title={Online learning over a finite action set with limited switching},
  author={Altschuler, Jason and Talwar, Kunal},
  booktitle={Proc. of COLT},
  year={2018},
}

@inproceedings{adamskiy2012closer,
  title={A closer look at adaptive regret},
  author={Adamskiy, Dmitry and Koolen, Wouter M and Chernov, Alexey and Vovk, Vladimir},
  booktitle={Proc. of ALT},
  year={2012},
}

@inproceedings{xu2024online,
  title={Online non-convex learning in dynamic environments},
  author={Xu, Zhipan and Zhang, Lijun},
  booktitle={Proc. of NeurIPS},
  pages={51930--51962},
  year={2024}
}
\bibliographystyle{abbrvnat}

\clearpage
\appendix

\section*{\Large Appendix}
\addcontentsline{toc}{section}{Appendix}

This appendix contains the technical details omitted from the main text. We first derive the change-of-variables formula underlying the fresh perturbed minimizer distribution and recall the lazy-sampling coupling used to preserve the correct marginals while controlling switches. We then prove the lower bounds, the base-learner guarantees, the dyadic restart and interval-covering results, and finally the meta-learner reductions leading to the strongly adaptive and dynamic-regret guarantees.

\paragraph{Contents.}
\begin{itemize}
    \item \Cref{sec:lazy-sampling-sov}: change of variables and lazy sampling.
    \item \Cref{sec:lower-bounds}: lower bounds for restart-based FPRLL.
    \item \Cref{sec:base-learner-proofs}: proofs for the base learner.
    \item \Cref{sec:dyadic-intervals}: dyadic interval covering.
    \item \Cref{sec:meta-learner-proofs}: proofs for the meta-learner and reductions.
\end{itemize}

\section{Change of variable and lazy sampling}
\label{sec:lazy-sampling-sov}
\subsection{Change of variable formula}
\begin{lemma}[Density of the fresh minimizer]
\label{lem:fresh-minimizer-density}
Let
\[
\tilde{\vec x}_{t}
~\doteq~
\arg\min_{\vec x\in\mathcal X}
\Bigl\{F_{t-1}(\vec x)+\langle \vec p_t,\vec x\rangle\Bigr\},
\qquad
\vec p_t\sim\Lap(\mu),
\]
where
\[
\nu(\vec p)=\frac{1}{(2\mu)^d}\exp\left(-\frac{\|\vec p\|_1}{\mu}\right).
\]
Then $\tilde{\vec x}_t$ has density
\begin{align}
\cQ_t(\vec x)
&=
\nu\left(-\nabla F_{t-1}(\vec x)\right)\,
\left|\det\left(-\nabla^2 F_{t-1}(\vec x)\right)\right|,
\qquad \vec x\in \mathrm{int}(\mathcal X).
\end{align}
\end{lemma}

\begin{proof}
By optimality of $\tilde{\vec x}_t$,
\[
\nabla F_{t-1}(\tilde{\vec x}_t)+\vec p_t=0,
\]
hence
\[
\vec p_t=-\nabla F_{t-1}(\tilde{\vec x}_t).
\]
Define
\[
T_t(\vec x)\doteq -\nabla F_{t-1}(\vec x).
\]

Since $F_{t-1}$ is strongly convex, its conjugate $F_{t-1}^*$ is differentiable and
\[
\nabla F_{t-1}^*=(\nabla F_{t-1})^{-1}.
\]
Therefore
\[
T_t^{-1}(\vec p)=\nabla F_{t-1}^*(-\vec p).
\]
So $\tilde{\vec x}_t=T_t^{-1}(\vec p_t)$.

Applying the change-of-variables formula to the map $T_t$, the density of
$\tilde{\vec x}_t$ is
\[
\cQ_t(\vec x)
=
\nu(T_t(\vec x))
\left|\det DT_t(\vec x)\right|.
\]
Using
\[
T_t(\vec x)=-\nabla F_{t-1}(\vec x),
\qquad
DT_t(\vec x)=-\nabla^2 F_{t-1}(\vec x),
\]
we obtain
\[
\cQ_t(\vec x)
=
\nu\left(-\nabla F_{t-1}(\vec x)\right)\,
\left|\det\left(-\nabla^2 F_{t-1}(\vec x)\right)\right|,
\]
as claimed.
\end{proof}

\subsection{Lazy sampling}
\label{sec:lazy-sampling}

We recall a standard lazy-sampling routine for realizing a maximal coupling between two distributions. We use it only through the marginal-preservation and switching identities in \cref{lem:lazy-sampling}; the results (and their proof) are standard and can be found in \cite{sherman2021lazy}.

\begin{algorithm}[h]
\caption{\textsc{LazySample}}
\label{alg:lazy-sample}
\begin{algorithmic}[1]
\State \textbf{Input:} current action $\vec x$, current distribution $\cQ$, next distribution $\cP$
\State \textbf{Output:} next action $\vec y$
\State Sample $z \sim \mathrm{Unif}[0,\,\cQ(\vec x)]$
\If{$z < \cP(\vec x)$} \Comment{Keep the same action}
    \State Set $\vec y = \vec x$
\Else \Comment{Sample a new action}
    \Repeat
        \State Sample $\vec y \sim \cP$
        \State Sample $z' \sim \mathrm{Unif}[0,\,\cP(\vec y)]$
    \Until{$z' > \cQ(\vec y)$}
\EndIf
\State \textbf{return} $\vec y$
\end{algorithmic}
\end{algorithm}

\begin{lemma}[Standard lazy-sampling identities]
\label{lem:lazy-sampling}
Let $\vec X \sim \cQ$, and let
\[
\vec Y \sim \textsc{LazySample}(\vec X,\cQ,\cP).
\]
Then:
\begin{align}
\vec Y &\sim \cP, \label{eq:lazy-marginal}\\
\Pr(\vec Y \neq \vec X) &= \|\cP-\cQ\|_{\mathrm{TV}}. \label{eq:lazy-tv}
\end{align}
\end{lemma}

\paragraph{Correspondence with \cref{main-alg}.}
\Cref{main-alg} is precisely \textsc{LazySample} instantiated with the consecutive
FPRLL marginals $(\cQ_t,\cQ_{t+1})$. In this correspondence, the current
state is $\vec x_t$, the current and next distributions are $\cQ_t$ and
$\cQ_{t+1}$, and the output is $\vec x_{t+1}$. The ``keep'' test in
\cref{main-alg} matches the first branch of \cref{alg:lazy-sample}. If that test
fails, the algorithm must draw a fresh sample from $\cQ_{t+1}$; in our setting
this is implemented by drawing $\vec p\sim\Lap(\mu)$ and solving the perturbed
optimization problem, whose minimizer has law $\cQ_{t+1}$ by definition. Finally,
\cref{lem:fresh-minimizer-density} supplies a closed-form expression for these
densities, so the comparisons involving $\cQ_t(\cdot)$ and $\cQ_{t+1}(\cdot)$
required by lazy sampling can be evaluated pointwise.

\section{Lower bounds}
\label{sec:lower-bounds}

First, we need the following tool, which lower bounds FPRLL dynamic regret without restarts: 

\begin{theorem}
\label{thm:lb-FPRLL-dyn-robust-sigmat}
Consider the one-dimensional problem on $\X=[-1,1]$ with linear losses
\[f_t(x)=g_t x, \qquad g_t\in\{-1,+1\}. \]
Let the learner predict at each round
\begin{align}
x_t
\doteq \argmin_{x\in(-1,1)} \Big\{ F_t (x) + \dtp{G_0+p_t}{x} \Big\} = \argmin_{x\in(-1,1)} \Big\{ \dtp{G_{t-1}+p_t}{x} + r_t(x) + b(x) \Big\},
\end{align}
where \(G_{t-1}=G_0+\sum_{i=1}^{t-1}g_i\), for an arbitrary \(G_0\in\mathbb R\),
\[
b(x)=-\frac{1}{M_\gamma}\log(1-x^2)
\]
is the logarithmic barrier for \(\X\), \(r_t(x)=\frac{\sigma_t}{2}x^2\) with \(\sigma_t\ge 0\), and \(p_t\) are arbitrary symmetric perturbations.

Then, there exists an oblivious loss sequence and a one-switch comparator sequence such that
\[
\mathbb{E}\left[\sum_{t=1}^T \bigl(f_t(x_t)-f_t(u_t)\bigr)\right]
\;\ge\;
\frac{T}{2}.
\]
\end{theorem}

\begin{proof}
Fix $T=2n$ and consider the oblivious loss sequence
\[
g_t =
\begin{cases}
+1,& t\le n,\\
-1,& t>n,
\end{cases}
\qquad
u_t = -g_t.
\]

Define the cumulative gradients
\[
G_t \doteq G_0 + \sum_{i=1}^t g_i.
\]

\paragraph{Step 1: Structure of the prediction.}
For each round $t$, define
\[
H_t(x) \doteq r_t'(x) + b'(x)
= \sigma_t x + \frac{2x}{M_\gamma(1-x^2)}
= x\left(\sigma_t + \frac{2}{M_\gamma(1-x^2)}\right),
\qquad x\in(-1,1).
\]
Since $\sigma_t \ge 0$,
\[
H_t'(x)
=
\sigma_t + \frac{2(1+x^2)}{M_\gamma(1-x^2)^2}
> 0,
\]
so $H_t$ is strictly increasing. Moreover, $H_t$ is odd and satisfies
\[
\lim_{x\to -1^+} H_t(x) = -\infty,
\qquad
\lim_{x\to 1^-} H_t(x) = +\infty.
\]
Thus $H_t$ is a bijection from $(-1,1)$ onto $\mathbb{R}$, and its inverse $H_t^{-1}$ is odd and strictly increasing.

By first-order optimality of the FTRL objective,
\[
H_t(x_t) = -(G_{t-1}+p_t),
\]
and therefore
\[
x_t = -\,H_t^{-1}(G_{t-1}+p_t).
\]

\paragraph{Step 2: Sign of the expectation.}
Define
\[
m_t(c) \doteq \mathbb{E}\bigl[H_t^{-1}(c+p_t)\bigr].
\]
Since $H_t^{-1}$ is increasing, $m_t$ is increasing. Since $H_t^{-1}$ is odd and $p_t$ is symmetric, $m_t$ is odd:
\[
m_t(-c) = -m_t(c).
\]
Hence
\[
c \ge 0 \;\Rightarrow\; m_t(c)\ge 0,
\qquad
c \le 0 \;\Rightarrow\; m_t(c)\le 0.
\]
Using $x_t = -H_t^{-1}(G_{t-1}+p_t)$,
\[
\mathbb{E}[x_t] = -\,m_t(G_{t-1}),
\]
which implies
\[
G_{t-1}\ge 0 \;\Rightarrow\; \mathbb{E}[x_t]\le 0,
\qquad
G_{t-1}\le 0 \;\Rightarrow\; \mathbb{E}[x_t]\ge 0.
\]

\paragraph{Step 3: Pairing argument.}
For each $k\in\{1,\dots,n\}$, pair round $k$ with round $2n-k+1$, and define
\[
R_k \doteq (f_k(x_k)-f_k(u_k))
+ (f_{2n-k+1}(x_{2n-k+1})-f_{2n-k+1}(u_{2n-k+1})).
\]
Since $g_k=+1$ and $g_{2n-k+1}=-1$,
\[
R_k = (1+x_k) + (1-x_{2n-k+1}).
\]
Both terms are nonnegative since $x_t\in[-1,1]$.

The cumulative gradients before the two rounds are
\[
A_k \doteq G_{k-1} = G_0 + k - 1,
\qquad
B_k \doteq G_{2n-k} = G_0 + k = A_k + 1.
\]

\paragraph{Step 4: Lower bounding each pair.}
We distinguish two cases.

\medskip
\noindent\textbf{Case 1: $A_k \le 0$.}
Then $\mathbb{E}[x_k]\ge 0$, so
\[
\mathbb{E}[1+x_k] \ge 1,
\]
and since $1-x_{2n-k+1}\ge 0$,
\[
\mathbb{E}[R_k] \ge 1.
\]

\medskip
\noindent\textbf{Case 2: $A_k > 0$.}
Then $B_k > 0$, so $\mathbb{E}[x_{2n-k+1}] \le 0$, and thus
\[
\mathbb{E}[1-x_{2n-k+1}] \ge 1.
\]
Since $1+x_k \ge 0$,
\[
\mathbb{E}[R_k] \ge 1.
\]

In all cases, $\mathbb{E}[R_k]\ge 1$ for every $k$. Summing over $k=1,\dots,n$,
\[
\mathbb{E}\left[\sum_{t=1}^T \bigl(f_t(x_t)-f_t(u_t)\bigr)\right]
= \sum_{k=1}^n \mathbb{E}[R_k]
\ge n
= \frac{T}{2}.
\]
\end{proof}
Note that starting with positive wave is not necessary, since what matters is the switch, as shown next.
\begin{corollary}
\label{cor:lb-FPRLL-dyn-robust-sigmat-reversed}
Under the assumptions of Theorem~\ref{thm:lb-FPRLL-dyn-robust-sigmat}, let $T=2n$ and consider the oblivious loss sequence
\[
g_t =
\begin{cases}
-1,& t\le n,\\
+1,& t>n,
\end{cases}
\qquad
u_t=-g_t.
\]
Then
\[
\mathbb{E}\left[\sum_{t=1}^T \bigl(f_t(x_t)-f_t(u_t)\bigr)\right]
\;\ge\;
\frac{T}{2}.
\]
\end{corollary}

\begin{proof}
The proof is identical to that of
Theorem~\ref{thm:lb-FPRLL-dyn-robust-sigmat}, except that in the pairing step $3$, we have
\[
R_k
\doteq
(f_k(x_k)-f_k(u_k))
+
(f_{2n-k+1}(x_{2n-k+1})-f_{2n-k+1}(u_{2n-k+1})).
\]
Since now $g_k=-1$ and $g_{2n-k+1}=+1$,
\[
R_k=(1-x_k)+(1+x_{2n-k+1}).
\]
The cumulative gradients before the two rounds are
\[
A_k \doteq G_{k-1}=G_0-(k-1),
\qquad
B_k \doteq G_{2n-k}=G_0-k=A_k-1.
\]
If $A_k\ge 0$, then $\expec{}{x_k}\le 0$, so $\expec{}{1-x_k}\ge 1$.
If $A_k<0$, then $B_k<0$, so $\expec{}{x_{2n-k+1}}\ge 0$, and hence
$\expec{}{1+x_{2n-k+1}}\ge 1$.
Since both terms are always nonnegative, in either case $\expec{}{R_k}\ge 1$.
Summing over $k=1,\dots,n$ yields the claim.
\end{proof}

Now we are ready to prove \cref{thm:dyadic-restart-lb}:
\begin{proof}
The adversary samples a half-gadget length $h$ uniformly from $\mathcal H$.
It then builds the loss sequence by concatenating $\tau$ gadgets of length
$2h$, with alternating orientation:
for odd $m\in\{1,\dots,\tau\}$, the $m$-th gadget is
\[
(\underbrace{+1,\dots,+1}_{h\text{ rounds}},
 \underbrace{-1,\dots,-1}_{h\text{ rounds}}),
\]
and for even $m\in\{1,\dots,\tau\}$, the $m$-th gadget is
\[
(\underbrace{-1,\dots,-1}_{h\text{ rounds}},
 \underbrace{+1,\dots,+1}_{h\text{ rounds}}).
\]
Since $h\le \eta T/\tau$ and $\eta<1/2$, the total length of these $\tau$
gadgets is at most
\[
2h\tau \le 2\eta T \le T,
\]
so the construction fits in the horizon. For the remaining rounds
$t>2h\tau$, define
\[
g_t \doteq g_{2h\tau}.
\]
Finally, set the comparator to oppose the losses (for each realized randomization):
\[
u_t \doteq -g_t \qquad \text{for all } t=1,\dots,T.
\]

\paragraph{Comparator path length.}
Inside each gadget, the comparator switches exactly once at the midpoint. The padded rounds are constant and create no further switches. Therefore
\[
P_T(u)=2\tau.
\]

Note that for every round $t$, since $u_t=-g_t$ and $x_t\in[-1,1]$,
\[
f_t(x_t)-f_t(u_t)
=
g_t x_t - g_t u_t
=
g_t x_t + 1
\ge 0.
\]
Hence
$\expec{}{\R_T}\ge 0$ for every realization of the randomness.

\paragraph{Choose the dyadic scale.}
Now fix any restart period $k$ satisfying $1/\eta \le k \le T/\tau$, and define
\[
h^\star \doteq 2^{\lfloor \log_2(\eta k)\rfloor}.
\]
Since $\eta k\ge 1$, we have $h^\star\ge 1$, and since
$\eta k \le \eta T/\tau$, we have $h^\star\in\mathcal H$. By construction,
\[
\frac{\eta k}{2} < h^\star \le \eta k.
\]
This is key: the gadget half-length is within a constant
factor of (a fraction of) the restart period: 
\[
2h^\star \le 2\eta k < k.
\]
Thus each gadget is strictly shorter than one restart block.

Let $A \doteq \{h=h^\star\}.$
Because $h$ is drawn uniformly from $\mathcal H$, 
\[
\Pr(A)=\frac{1}{J+1}, \quad J \doteq \left\lfloor \log_2\left(\eta T/\tau\right)\right\rfloor .
\]

\paragraph{Number of uncut gadgets}
Condition on the event $A$.
The active gadget region has total length
\[
2h^\star\tau \le 2\eta k\tau.
\]
Since restart boundaries occur every $k$ rounds, the number of restart
boundaries falling inside this active region is at most
\[
\left\lfloor \frac{2h^\star\tau}{k}\right\rfloor
\le
\frac{2h^\star\tau}{k}
\le
2\eta\tau.
\]

Because every gadget has length $2h^\star<k$, a single restart boundary can
intersect the interior of at most one gadget, and a single gadget can contain
at most one restart boundary. Therefore at most $2\eta\tau$ gadgets are cut by
a restart, and at least
\[
\tau-2\eta\tau=(1-2\eta)\tau
\]
gadgets are fully contained in a single restart block.

\paragraph{Regret from each uncut gadget.}
Consider any gadget that is not cut by a restart. Over that gadget, the learner
runs continuously with no reset. Reindex the gadget as a fresh $2h^\star$-round
instance. The state at the beginning of the gadget simply becomes the initial
gradient $G_0$ for that local instance.

For an odd gadget, the loss sequence is precisely the two-phase sequence used in the proof of
Thm.~\ref{thm:lb-FPRLL-dyn-robust-sigmat}. Hence that gadget contributes
expected dynamic regret at least
$\frac{2h^\star}{2}=h^\star.$
For an even gadget, the same lower bound follows from Corollary~\ref{cor:lb-FPRLL-dyn-robust-sigmat-reversed}.

Summing over the at least $(1-2\eta)\tau$ uncut gadgets yields
\[
\expec{}{\R_T\mid A}
\;\ge\;
(1-2\eta)\tau\, h^\star
\geq
\frac{\eta(1-2\eta)}{2}\,k\tau.
\]

We expand the unconditional expectation:
\[
\expec{}{\R_T}
=
\Pr(A)\,\expec{}{\R_T\mid A}
+
\Pr(A^c)\, \expec{}{\R_T\mid A^c}.
\]
Since $\expec{}{\R_T}\ge 0$, both conditional expectations
are nonnegative, and in particular
\[
\Pr(A^c)\, \expec{}{\R_T\mid A^c}\ge 0.
\]
Therefore we may drop this second term and obtain the lower bound
\[
\expec{}{\R_T}
\;\ge\;
\Pr(A)\,\expec{}{\R_T\mid A}.
\]
Using $\Pr (A)=1/(J+1)$ and the conditional lower bound above gives
\[
\expec{}{\R_T}
\;\ge\;
\frac{1}{J+1}\cdot \frac{\eta(1-2\eta)}{2}\,k\tau.
\]

Finally, since $\eta\in(0,1/2)$ is a fixed constant,
\[
J+1
=
1+\left\lfloor \log_2\left(\eta \frac{T}{\tau}\right)\right\rfloor
=
\Theta(\log(T/\tau)).
\]
Hence
$
\expec{}{\R_T}
=
\Omega\left(\frac{k\tau}{\log(T/\tau)}\right).
$
This completes the proof.
\end{proof}

\begin{lemma}[Sherman branch]
\label{lem:sherman-branch}
Suppose that a fresh run of FPRLL\((k)\) makes at most
\(C\sqrt{k}\) expected switches on each length-\(k\) block,
for a constant \(C>0\). Then there exists an oblivious distribution over
\(T\)-round loss sequences and a fixed comparator \(u^\star\in[-1,1]\) such
that
\[
\expec{}{\Rind_T(u^\star,\ldots,u^\star)}
=
\Omega\left(\frac{T}{\sqrt{k}}\right).
\]
\end{lemma}

\begin{proof}
We use the following consequence of \citet[Thm.~4]{shermanArxiv}: for a
length-\(k\) online linear problem on \([-1,1]\), there exist a stochastic
loss distribution \(\mathcal D_k\) and a comparator \(u^\star\in[-1,1]\),
fixed as a minimizer of the expected loss under \(\mathcal D_k\), such that
any algorithm making \(\bigo(\sqrt{k})\) expected switches incurs
\(\Omega(\sqrt{k})\) expected static regret against \(u^\star\).

Construct the \(T\)-round adversary by drawing, independently on each full
restart block, a fresh length-\(k\) sample from this same distribution
\(\mathcal D_k\). Although the realized losses differ from block to block,
the comparator remains the same \(u^\star\), since \(u^\star\) is defined by \(\mathcal D_k\), not by the realized sample.

Thus every full block contributes \(\Omega(\sqrt{k})\) expected static
regret. Since \(\Rind_T\) only adds the learner's nonnegative indicator
switching cost, the same lower bound holds for \(\Rind_T\). Therefore
\[
\expec{}{\Rind_T(u^\star,\ldots,u^\star)}
\ge
\left\lfloor \frac{T}{k}\right\rfloor \Omega(\sqrt{k})
=
\Omega\left(\frac{T}{\sqrt{k}}\right),
\]
where \(k\le T\). The comparator is stationary throughout, so \(S_T=P_T=0\).
\end{proof}

\begin{proposition}[Combined lower bound]
\label{lem:combined-lb}
Consider the following two full-horizon adversaries:
\begin{enumerate}
    \item the adversary from \cref{thm:dyadic-restart-lb}, together with
    comparator sequence \(u^{\mathrm{gad}}_{1:T}\), for which
    \[
    S_T(u^{\mathrm{gad}}_{1:T})=\tau,
    \qquad
    P_T(u^{\mathrm{gad}}_{1:T})=2\tau,
    \]
    and
    \[
    \expec{}{\Rind_T(u^{\mathrm{gad}}_{1:T})}
    =
    \Omega\left(\frac{k\tau}{\log(T/\tau)}\right);
    \]

    \item the adversary from \cref{lem:sherman-branch}, together with the
    stationary comparator sequence
    \(u^{\mathrm{sh}}_{1:T}\equiv u^\star\), for which
    \[
    S_T(u^{\mathrm{sh}}_{1:T})=0,
    \qquad
    P_T(u^{\mathrm{sh}}_{1:T})=0,
    \]
    and
    \[
    \expec{}{\Rind_T(u^{\mathrm{sh}}_{1:T})}
    =
    \Omega\left(\frac{T}{\sqrt{k}}\right).
    \]
\end{enumerate}

For the fixed restart period \(k\), define a mixed adversary that, at time
\(0\), chooses one of these two branches uniformly at random and then plays
the chosen branch for the entire horizon. The resulting mixed adversary may
depend on \(k\) through the Sherman branch.
Let \(u_{1:T}\) denote the comparator sequence associated with the
realized branch. Then
\[
\expec{}{\Rind_T(u_{1:T})}
=
\Omega\left(
\frac{k\tau}{\log(T/\tau)}
+
\frac{T}{\sqrt{k}}
\right).
\]
\end{proposition}

\begin{proof}
Conditioning on the branch chosen by the adversary gives
\[
\expec{}{\Rind_T(u_{1:T})}
=
\frac12
\expec{}{\Rind_T(u^{\mathrm{gad}}_{1:T})}
+
\frac12
\expec{}{\Rind_T(u^{\mathrm{sh}}_{1:T})}.
\]
Applying \cref{thm:dyadic-restart-lb} to the first term and
\cref{lem:sherman-branch} to the second term yields the claim.
\end{proof}

\begin{proof}
Condition on $C$. The $C{=}0$ branch is $k_0$-independent, and
\cref{thm:dyadic-restart-lb} gives $\Omega(k\tau/\log(T/\tau))$. For
$C{=}1$, condition further on $k_0$: on $\{k_0=k\}$,
\cref{lem:sherman-branch} yields $\Omega(T/\sqrt{k})$; on $\{k_0\neq k\}$,
$\E[\Rind_T]\ge 0$ since $u^\star_{k_0}$ is the per-round expected-loss
minimizer of $\cD^{\mathrm{sh}}_{k_0}$. As $\Pr(k_0=k)=1/|\mathcal H|=
\Theta(1/\log T)$, the $C{=}1$ branch contributes $\Omega(T/(\sqrt{k}\log T))$.
\end{proof}

\section{Proofs for the base learner}
\label{sec:base-learner-proofs}

To prove \cref{thm:agnostic-dynamic}, we establish four intermediate lemmas and then combine them.

\begin{itemize}
    \item \Cref{lem:strong-ftrl-perturbed-regret} gives the main regret decomposition, splitting the expected dynamic regret into a \emph{static} term $\mathbf{(I)}$ and a \emph{dynamic} term $\mathbf{(II)}$. This extends the decomposition in \cite{pruning-icml} to our perturbed, barrier-regularized setting.
    
    \item \Cref{lem:static-part} bounds the static term $\mathbf{(I)}$, explicitly accounting for the effect of the Laplace perturbations.
    
    \item \Cref{lem:dynamic-part-2} bounds the dynamic term $\mathbf{(II)}$, capturing the dependence on comparator drift through the barrier geometry.
    
    \item \Cref{lem:switching-part-3} controls the total variation distance between consecutive sampling distributions, which we use to bound the switching-cost contribution.
\end{itemize}

Among these ingredients, only the total-variation bound in \cref{lem:switching-part-3} follows the general approach of \cite{sherman2021lazy}; the other three lemmas are specific to our setting.

\begin{lemma}
  \label{lem:strong-ftrl-perturbed-regret}
Let $\{f_{t}(\cdot)\}_{t=1}^T$ be an arbitrary set of functions. Let $\vec p_t \sim \text{Lap}(\mu)$. Assume that for all $t$, the fresh perturbed minimizer 
\[
\tilde{\vec{x}}_{t+1} \doteq \argmin_{\vec x} F_t(\vec{x}) + \dtp{\vec p_{t}}{\vec x}
\] is well-defined. 
Define the loss part of $\Rind_T(\comseq)$ as 
\[
\R_T(\comseq) = \sumT \left(f_t(\vec x_t) - f_t(\vec u_t)\right)
\], where $\vec x_t$ are the output of \cref{main-alg}.
Then, the algorithm that selects the actions $\tilde{\vec{x}}_{t+1}, \forall t $ achieves the following dynamic regret bound:
\begin{equation}
    \expec{}{\mathcal{R}_T} \leq \overbrace{\sum_{t=1}^T F_t(\vec {\ti x}_t) - F_t(\vec {\ti{x}}_{t+1})}^{\mathbf{(I)}}\ 
    +\ \underbrace{\sum_{t=1}^{T-1} F_t(\vec u^\gamma_{t+1}) - F_t(\vec u^\gamma_{t})}_{\mathbf{(II)}} \ +\ r(\vec{u}^\gamma_1) +1+ G\gamma R\, T.
\end{equation}
\end{lemma}

\begin{proof}[Proof of \cref{lem:strong-ftrl-perturbed-regret}]
    \begin{align*}
  &\sumT f_t(\ti{\vec x}_t) - \sumT f_t(\vec u_t) - \dtp{\vec p_{T}}{\vec u^\gamma_T} 
  \\
  & = \sumT f_t(\ti{\vec x}_t) - \sumT f_t(\vec u^\gamma_t) - \dtp{\vec p_{T}}{\vec u^\gamma_T} \;+\;\sumT \big(\overbrace{f_t(\vec u^\gamma_t) - f_t(\vec u_t)}^{\text{$a_t$}}\big) 
  \\
  &= \sumT \left(F_{t}(\ti{\vec x}_t) - F_{t-1}(\ti{\vec x}_t)\right) - \left( \sumT \left(F_t(\vec u^\gamma_t) - F_{t-1} (\vec u^\gamma_t)\right)+ \dtp{\vec p_{T}}{\vec u^\gamma_T}\right) \,+\,a_{1:T}
  \\
  &= \sumT F_t(\ti{\vec x}_t) - \sumT F_{t-1}(\ti{\vec x}_t) - \left( \sumT F_t(\vec u^\gamma_t)+ \dtp{\vec p_{T}}{\vec u^\gamma_T} - \sumT F_{t-1}(\vec u^\gamma_t)\right) + a_{1:T}
  \\
  &= \sumT F_t(\ti{\vec x}_t) - \sum_{t=0}^{T-1} F_t(\ti{\vec x}_{t+1}) - \left( F_{T}(\vec u^\gamma_T)+ \dtp{\vec p_{T}}{\vec u^\gamma_T} + \sum_{t=1}^{T-1} F_t(\vec u^\gamma_t) - \sum_{t=0}^{T-1} F_t(\vec u^\gamma_{t+1})\right) 
  \\
  &\qquad + a_{1:T}
  \\
  &\leq \sumT F_t(\ti{\vec x}_t) - \sum_{t=1}^{T-1} F_t(\ti{\vec x}_{t+1}) - \left( F_{T}(\vec x_{T+1})+ \dtp{\vec p_{T}}{\vec x_{T+1}}+\sum_{t=1}^{T-1} F_t(\vec u^\gamma_t) - \sum_{t=1}^{T-1} F_t(\vec u^\gamma_{t+1})\right) 
  \\
  &\qquad + a_{1:T} + f_0(\vec u^\gamma_1)
  \\
  &\leq \sumT F_t(\ti{\vec x}_t) - \sum_{t=1}^{T} F_t(\ti{\vec x}_{t+1}) - \left(\sum_{t=1}^{T-1} F_t(\vec u^\gamma_t) - \sum_{t=1}^{T-1} F_t(\vec u^\gamma_{t+1})\right) - \dtp{\vec p_{T}}{\vec x_{T+1}} + a_{1:T} + r(\vec u^\gamma_1) + 1.
\end{align*}
The first inequality uses the update-rule optimality
\[
F_T(\vec x_{T+1})+ \dtp{\vec p_{T}}{\vec x_{T+1}}\leq F_T(\vec u^\gamma_{T}) + \dtp{\vec p_{T}}{\vec u^\gamma_T},
\]
and $-F_0(\tilde{\vec x}_1) \le 0$ (since $r \ge 0$ and $b_\gamma \ge 0$ on $\mathcal{X}$).
The second inequality uses $b_\gamma \le 1$ on $\mathcal{X}^\gamma$, giving $f_0(\vec u^\gamma_1) \le r(\vec u^\gamma_1) + 1$.
Overall
\begin{align*}
    \sumT &\left( f_t(\ti{\vec x}_t) - f_t(\vec u_t)\right)  - \dtp{\vec p_{T}}{\vec u^\gamma_T}
    \\
    &\leq \sumT \left( F_{t}(\ti{\vec x}_t) - F_{t}(\ti{\vec x}_{t+1}) \right)+ \sum_{t=1}^{T-1} \left( F_{t}(\vec u^\gamma_{t+1}) -  F_{t}(\vec u^\gamma_t) \right) +r(\vec u_1) +1 + a_{1:T} - \dtp{\vec{p}_{T}}{\vec x_{T+1}}.
\end{align*}
Since the terminal perturbation and terminal point are used only for
analysis and are never played by the algorithm, we set
$\vec p_T=\vec 0$ and choose
$\vec x_{T+1}=\argmin_{\vec x\in\mathcal X} F_T(\vec x)$.
Then the preceding optimality inequality remains valid, and both terminal inner-product terms vanish.
In addition, by Lipschitzness:
\begin{align}
    a_{1:T} = \sumT f_t(\vec u^\gamma_t) - f_t(\vec u_t) \leq \sumT G\|\vec u^\gamma_t - \vec u_t\| \leq G\gamma R\, T.
\end{align}

Finally: Note since $\vec x_t$ and $\tilde{\vec x}_t$ have the same density for all $t$, we have 
\begin{align}
  \expec{}{f_t(\vec x_t)} = \expec{}{f_t(\ti {\vec x_t})}
\end{align}
The above property is guaranteed by the lazy sampling procedure: the fresh minimizer and the (potentially) kept one share the same density.

Hence \begin{align*}
    \expec{}{\R_T} &= \expec{}{\sumT f_t(\vec x_t) - f_t(\vec u_t)} = \expec{}{\sumT f_t(\ti{\vec x}_t) - f_t(\vec u_t)} 
    \\
    &\leq  \sumT \expec{}{\left(F_{t}(\ti{\vec x}_t) - F_{t}(\ti{\vec x}_{t+1})  \right)} \;+\; \sum_{t=1}^{T-1} \left( F_{t}(\vec u^\gamma_{t+1}) -  F_{t}(\vec u^\gamma_t) \right) \;+\;r(\vec u_1) +1+ G\gamma R\, T.
\end{align*} 
\end{proof}

\begin{lemma} Under the same conditions as \cref{lem:strong-ftrl-perturbed-regret}, if we further assume that $F_t$ is $\sigma$-strongly convex for all $t$, then
\label{lem:static-part}
\[
\mathbf {(I)} \le \tfrac{G^2}{2\sigma}\,T \;+\; R \sqrt{2d}\,\mu.
\], where $\mu$ is the scale parameter of the Laplace distribution.
\end{lemma}

\begin{proof}[Proof of \cref{lem:static-part}]
For each $t$, define
\begin{align}
    \tilde F_t(\cdot) \doteq F_t(\cdot) + \dtp{\vec p_{t-1}}{\cdot}.
    \label{eq:aux-ti-F}
\end{align}
Then,
\begin{align}
    &F_t(\ti{\vec x}_t)-F_t(\ti{\vec x}_{t+1})
    \\
    &= F_t(\ti{\vec x}_t)+\dtp{\vec p_{t-1}}{\ti{\vec x}_t}
    -\Big(F_t(\ti{\vec x}_{t+1})+\dtp{\vec p_{t-1}}{\ti{\vec x}_{t+1}}\Big)
    +\dtp{\vec p_{t-1}}{\ti{\vec x}_{t+1}-\ti{\vec x}_t}
    \\
    &= \tilde F_t(\ti{\vec x}_t)-\tilde F_t(\ti{\vec x}_{t+1})
    +\dtp{\vec p_{t-1}}{\ti{\vec x}_{t+1}-\ti{\vec x}_t}.
    \label{eq:static-decomp}
\end{align}

By strong convexity of $\tilde F_t$, for
\[
\vec s_t \doteq \nabla F_t(\ti{\vec x}_t)+\vec p_{t-1},
\]
we have
\begin{align}
    \tilde F_t(\ti{\vec x}_t)-\tilde F_t(\ti{\vec x}_{t+1})
    \;\le\;
    \dtp{\vec s_t}{\ti{\vec x}_t-\ti{\vec x}_{t+1}}
    -\frac{\sigma}{2}\delta_t^2,
    \label{eq:sc-actions}
\end{align}
where
\[
\delta_t \doteq \|\ti{\vec x}_{t+1}-\ti{\vec x}_t\|.
\]

The optimality condition of the update step at $\ti{\vec x}_t$ gives
\[
0=\nabla F_{t-1}(\ti{\vec x}_t)+\vec p_{t-1},
\]
and therefore
\[
\nabla f_t(\ti{\vec x}_t)
=
\nabla F_t(\ti{\vec x}_t)+\vec p_{t-1}
=
\vec s_t.
\]
Thus, substituting into \eqref{eq:sc-actions} yields
\begin{align}
    \tilde F_t(\ti{\vec x}_t)-\tilde F_t(\ti{\vec x}_{t+1})
    &\le
    \dtp{\vec g_t}{\ti{\vec x}_t-\ti{\vec x}_{t+1}}
    -\frac{\sigma}{2}\delta_t^2
    \\
    &\le
    G\delta_t-\frac{\sigma}{2}\delta_t^2\le
    \frac{G^2}{2\sigma},
    \label{eq:static1}
\end{align}
where $\vec g_t \doteq \nabla f_t(\ti{\vec x}_t)$. The second inequality follows from Cauchy--Schwarz and Lipschitzness, and the third from the elementary inequality
\[
ax-\frac{b}{2}x^2 \le \frac{a^2}{2b},
\qquad a,b>0.
\]

Combining \eqref{eq:static-decomp} and \eqref{eq:static1}, and summing over $t=1,\dots,T$, we obtain the following pathwise bound:
\begin{align}
    \sum_{t=1}^T \Big(F_t(\ti{\vec x}_t)-F_t(\ti{\vec x}_{t+1})\Big)
    \le
    \frac{G^2}{2\sigma}T
    +
    \sum_{t=1}^T \dtp{\vec p_{t-1}}{\ti{\vec x}_{t+1}-\ti{\vec x}_t}.
    \label{eq:pathwise-static}
\end{align}

\vspace{2mm}

Now let $\mathcal A$ denote the original perturbation law, under which the vectors $\{\vec p_t\}_{t=0}^{T-1}$ are i.i.d.\ Laplace. Let $\mathcal B$ be any other law on the perturbations such that each $\vec p_t$ has the same marginal distribution as under $\mathcal A$.

Recall that $\ti{\vec x}_t$, by definition, depends only on the randomness of $\vec p_{t-1}$. Therefore, as long as the perturbation marginals coincide, the induced marginal distribution of each $\ti{\vec x}_t$ is identical under $\mathcal A$ and $\mathcal B$. Hence,
\begin{align}
    &\sum_{t=1}^T \expec{\mathcal A}{F_t(\ti{\vec x}_t)-F_t(\ti{\vec x}_{t+1})} =
    \sum_{t=1}^T \expec{\mathcal B}{F_t(\ti{\vec x}_t)-F_t(\ti{\vec x}_{t+1})}.
    \label{eq:AB-equality}
\end{align}

We now choose $\mathcal B$ to simplify the perturbation term in \eqref{eq:pathwise-static}. Namely, let $\mathcal B$ be the law defined by
\begin{align}
    \label{eq:alternative-B}
    \vec p_0 \sim \mathrm{Lap}(\mu),
    \qquad
    \vec p_t=\vec p_0
    \quad \forall t\in[T{-}1].
\end{align}
This law matches the single-time marginal distribution of each $\vec p_t$ under $\mathcal A$, since under $\mathcal A$ the perturbations $\vec p_t$ are i.i.d.\ Laplace.

Taking expectation of \eqref{eq:pathwise-static} under $\mathcal B$, and using \eqref{eq:AB-equality}, gives
\begin{align}
    \mathbf{(I)}
    &=
    \sum_{t=1}^T \expec{\mathcal A}{F_t(\ti{\vec x}_t)-F_t(\ti{\vec x}_{t+1})}
    \\
    &=
    \sum_{t=1}^T \expec{\mathcal B}{F_t(\ti{\vec x}_t)-F_t(\ti{\vec x}_{t+1})} \le
    \frac{G^2}{2\sigma}T
    +
    \sum_{t=1}^T \expec{\mathcal B}{\dtp{\vec p_{t-1}}{\ti{\vec x}_{t+1}-\ti{\vec x}_t}}.
    \label{eq:static-after-B}
\end{align}

To bound the remaining sum, we rewrite it as
\begin{align}
    \label{eq:sum-rewrite}
    \sum_{t=1}^T \dtp{\vec p_{t-1}}{\ti{\vec x}_{t+1}-\ti{\vec x}_t}
    =
    \dtp{\vec p_{T-1}}{\ti{\vec x}_{T+1}}
    -
    \dtp{\vec p_0}{\ti{\vec x}_1}
    -
    \sum_{t=2}^{T}\dtp{\vec p_{t-1}-\vec p_{t-2}}{\ti{\vec x}_{t}}.
\end{align}
Under $\mathcal B$, we have $\vec p_{t-1}-\vec p_{t-2}=0$ for all $t\ge 2$, hence
\begin{align}
    \sum_{t=1}^T \expec{\mathcal B}{\dtp{\vec p_{t-1}}{\ti{\vec x}_{t+1}-\ti{\vec x}_t}}
    &=
    \expec{\mathcal B}{\dtp{\vec p_0}{\ti{\vec x}_{T+1}-\ti{\vec x}_1}}.
\end{align}
Therefore,
\begin{align}
    \left|
    \sum_{t=1}^T \expec{\mathcal B}{\dtp{\vec p_{t-1}}{\ti{\vec x}_{t+1}-\ti{\vec x}_t}}
    \right|
    &\le
    \expec{\mathcal B}{\|\vec p_0\|\,\|\ti{\vec x}_{T+1}-\ti{\vec x}_1\|}
    \\
    &\le
    R\,\expec{}{\|\vec p_0\|}
    \\
    &\le
    R\sqrt{2d}\,\mu,
\end{align}
where the last inequality follows from the bound on the first absolute moment,
\[
\mathbb E\|\vec p_0\| \le \sqrt{2d}\,\mu,
\]
of a $d$-dimensional Laplace distribution with scale parameter $\mu$.

Substituting this into \eqref{eq:static-after-B} completes the proof.
\end{proof}

\begin{lemma}
  Under the same conditions as \cref{lem:static-part}, we have that
\label{lem:dynamic-part-2}
\begin{align}
  \mathbf{(II)}  \leq (T G + \sigma R + C G_c \sqrt{T})\;P_T
\end{align}
\end{lemma}

\begin{proof}[Proof of \cref{lem:dynamic-part-2}]
   From strong convexity
\begin{align}
    F_t(\vec u^\gamma_{t+1}) -  F_t(\vec u^\gamma_t) \leq \|\vec q_t\| \| \vec u^\gamma_{t+1} - \vec u^\gamma_{t}\| 
    - \frac{\sigma}{2} \| \vec u^\gamma_{t+1} - \vec u^\gamma_{t}\|^2, \label{eq:h-convex-bound}
\end{align}
where $\vec q_t = \grd F_t(\vec u^\gamma_{t+1})$.

To bound $\|\vec q_t\|$:
\begin{align}
    \grd F_t(\vec x) &= \grd f_{1:t}(\vec x) + \sigma \vec x + \grd b_\gamma(\vec x).
    \\
    &=\grd f_{1:t}(\vec x) + \sigma \vec x - \sum_{c=1}^C \frac{1}{M_\gamma s_c(\vec x)} \grd s_c(\vec x).
\end{align}

$\bullet$ For $\|\grd b_\gamma(\vec u_{t+1}^\gamma)\|$:

By definition, any $\vec {y}\in\mathcal{X}^\gamma$ can be written as $\vec {y}=(1-\gamma)\vec {x}$ for some $\vec {x}\in\mathcal{X}$, hence
$\vec {y}=(1-\gamma)\vec {x}+\gamma\vec {0}$.
Since $s_c$ is concave,
\[
s_c(\vec {y})
\ge (1-\gamma)s_c(\vec {x})+\gamma s_c(\vec {0})
\ge \gamma,
\]
where we used $s_c(\vec {x})\ge 0$ for $\vec {x}\in\mathcal{X}$ and the normalization $s_c(\vec {0})=1$.

The gradient of the scaled barrier is
\[
\nabla b_\gamma(\vec {x})
=
-\frac{1}{M_\gamma}\sum_{c=1}^C \frac{1}{s_c(\vec {x})}\,\nabla s_c(\vec {x}).
\]
Since $\|\nabla s_c(\vec {x})\|\le G_c$ for all vectors in the shrunk set $\vec {x}\in\mathcal{X}^\gamma$, then using
$s_c(\vec {x})\ge \gamma$ on $\mathcal{X}^\gamma$ yields
\[
\sup_{\vec {x}\in\mathcal{X}^\gamma}\|\nabla b_\gamma(\vec {x})\|
\le
\frac{1}{M_\gamma}\sum_{c=1}^C \frac{G_c}{\gamma}
=
\frac{1}{\gamma M_\gamma}\sum_{c=1}^C G_c \leq \frac{C G_c}{\gamma}.
\]

Finally, setting $\gamma = 1/\sqrt{T}$, we get $\|\grd b_\gamma(\vec u_{t+1}^\gamma)\| \leq C G_c \sqrt{T}$.

$\bullet$ For $\sigma \|\vec u_{t+1}^\gamma \| \leq (1-\gamma) \sigma R\leq \sigma R$

$\bullet$ For $\|\grd f_{1:t}(\vec u_{t+1}^\gamma)\|$, we can only do $\|\grd f_{1:t}(\vec u_{t+1}^\gamma)\| \leq t G \leq T G$.

Combining the three bounds gives
\[
\|\vec q_t\|
\le
TG+\sigma R+C G_c\sqrt T.
\]
Substituting this into \eqref{eq:h-convex-bound}, dropping the negative quadratic term, and summing over \(t=1,\dots,T-1\), we obtain
\[
\mathbf{(II)}
\le
(TG+\sigma R+C G_c\sqrt T)
\sum_{t=1}^{T-1}\|\vec u_{t+1}^\gamma-\vec u_t^\gamma\|.
\]
Since \(\sum_{t=1}^{T-1}\|\vec u_{t+1}^\gamma-\vec u_t^\gamma\| \leq P_T\), the result follows.
\end{proof}

\begin{lemma}
  \label{lem:switching-part-3}
Suppose the loss sequence $\{f_t\}$ consists of convex, $G$-Lipschitz, and $\beta$-smooth functions. Then, the total variation distance between $\cQ_t$ and $\cQ_{t+1}$ is bounded by
\[
\|\cQ_{t+1} - \cQ_t\|_{TV} \leq \frac{\beta d}{\sigma} + \frac{\sqrt{d}\,G}{\mu}.
\]
\end{lemma}

\begin{proof}[Proof of \cref{lem:switching-part-3}]
    By the change of variables formula, the density ratio satisfies
\[
\frac{\cQ_{t+1}(\vec x)}{\cQ_t(\vec x)}
=
\frac{\nu(-\nabla F_{t}(\vec x))}{\nu(-\nabla F_{t-1}(\vec x))}
\cdot
\Big|\frac{\det(-\nabla^2 F_{t}(\vec x))}{\det(-\nabla^2 F_{t-1}(\vec x))}\Big|\]

We bound each component separately.

\textit{Determinant ratio.}
Let $\lambda_{i,t}$ denote the $i$-th largest eigenvalue of the symmetric matrix $\nabla^2 F_t(\vec x)$. Because the newly revealed loss function $f_t$ is $\beta$-smooth, the spectral norm (operator norm) of its Hessian is bounded by $\beta$. By applying Weyl's inequality, which bounds the change in eigenvalues by the spectral norm of the difference, we obtain
\[
    |\lambda_{i,t} - \lambda_{i,t-1}| \leq \|\nabla^2 F_t(\vec x) - \nabla^2 F_{t-1}(\vec x)\|_{\mathrm{op}} = \|\nabla^2 f_t(\vec x)\|_{\mathrm{op}} \leq \beta
\]

Furthermore, since $F_t$ is also $\sigma$-strongly convex, all eigenvalues satisfy $\lambda_{i,t} \geq \sigma >0, \forall t$. Therefore,
\[
\frac{\det(\nabla^2 F_{t}(\vec x))}{\det(\nabla^2 F_{t-1}(\vec x))}
\leq
\prod_{i=1}^d
\left(\frac{\lambda_{i,t-1} + \beta}{\lambda_{i,t-1}}\right) \leq \prod_{i=1}^d \left(1+\frac{\beta}{\sigma}\right)
=
\left(1 + \frac{\beta}{\sigma}\right)^d
\leq
\exp\left(\frac{\beta d}{\sigma}\right).
\]

\textit{Density ratio.}
Using the $L_1$ structure of the Laplace distribution and the reverse triangle inequality,
\[
\frac{\nu(-\nabla F_{t}(\vec x))}{\nu(-\nabla F_{t-1}(\vec x))}
= \exp\left( \frac{\|\nabla F_{t-1}(\vec x)\|_1 - \|\nabla F_{t}(\vec x)\|_1}{\mu} \right)
\leq
\exp\left(
\frac{\|\nabla F_{t-1}(\vec x) - \nabla F_{t}(\vec x)\|_1}{\mu}
\right)
=
\exp\left(
\frac{\|\nabla f_t(\vec x)\|_1}{\mu}
\right).
\]

Since $f_t$ is $G$-Lipschitz,
\[
\|\nabla f_t(\vec x)\|_1
\leq
\sqrt{d}\,\|\nabla f_t(\vec x)\|_2
\leq
\sqrt{d}\,G.
\]
Thus,
\[
\frac{\nu(-\nabla F_{t}(\vec x))}{\nu(-\nabla F_{t-1}(\vec x))}\leq
\exp\left(
\frac{\sqrt{d}\,G}{\mu}
\right).
\]

Combining both bounds,
\[
\frac{\cQ_{t+1}(\vec x)}{\cQ_t(\vec x)}
\leq
\exp\left(
\frac{\beta d}{\sigma}
+
\frac{\sqrt{d}\,G}{\mu}
\right)
=
\exp(\epsilon),
\]
where
\[
\epsilon
\doteq \frac{\beta d}{\sigma}
+ \frac{\sqrt{d}\,G}{\mu}.
\]

Taking reciprocals gives
\[
\frac{\cQ_t(\vec x)}{\cQ_{t+1}(\vec x)}
\ge
\exp(-\epsilon),
\]
which implies
\[
\cQ_{t+1}(\vec x)-\cQ_t(\vec x)
\le
(1-e^{-\epsilon})\,\cQ_{t+1}(\vec x).
\]
Integrating over any measurable set $E$ yields
\[
\cQ_{t+1}(E)-\cQ_t(E)
\le
1-e^{-\epsilon}
\le
\epsilon.
\]

To also get $\cQ_{t}(E) - \cQ_{t+1}(E) \leq \epsilon$
we start the proof again but with the reciprocal bound $\cQ_t(\vec x)/\cQ_{t+1}(\vec x)$, and using $|\lambda_{i,t} - \lambda_{i,t-1}| \leq \beta$, we get
\[
\frac{\det(\nabla^2 F_{t-1}(\vec x))}{\det(\nabla^2 F_{t}(\vec x))}
\leq
\prod_{i=1}^d
\left(\frac{\lambda_{i,t} + \beta}{\lambda_{i,t}}\right)
\leq
\exp\left(\frac{\beta d}{\sigma}\right).
\]
The remaining proof steps are the same thereafter, and we get the same bound $\epsilon$. Hence, we have
\[
\|\cQ_{t+1} - \cQ_t\|_{TV}
\leq
\epsilon.
\]
\end{proof}

Now we are ready to prove the main theorem.
\begin{proof}[Proof of \cref{thm:agnostic-dynamic}]
By linearity of expectation,
\[
\expec{}{\Rind_T}
=
\expec{}{\R_T}
+
\lambda\sum_{t=2}^T \expec{}{\ind}.
\]
The bound on $\expec{}{\R_T}$ follows by combining \cref{lem:static-part} and \cref{lem:dynamic-part-2}. For the second term, the lazy sampling procedure implies
\[
\expec{}{\ind}
=
\Pr(\vec x_t \neq \vec x_{t-1})
=
\|\cQ_t - \cQ_{t-1}\|_{\mathrm{TV}}.
\]
Applying the bound on $\|\cQ_t - \cQ_{t-1}\|_{\mathrm{TV}}$ from \cref{lem:switching-part-3} yields the claimed result.
\end{proof}

\subsection{Proof of \cref{thm:dyadic-restart}}
\begin{proof}
Fix any dyadic scale \(k\in\mathcal H\), and partition the horizon into
\(m=T/k\) consecutive blocks
\[
B_i=\{(i-1)k+1,\dots,ik\},
\qquad i=1,\dots,m.
\]
For each block, define the within-block path budget
\[
P_i \doteq \sum_{t=(i-1)k+2}^{ik}\|\vec u_t- \vec u_{t-1}\|.
\]
Applying Theorem~\ref{thm:agnostic-dynamic} on each block with horizon \(k\) and
the given choices of \((\gamma,\mu,\sigma)\), we obtain
\[
\expec{} {\Rind_T(\comseq)}
\le
\sum_{i=1}^m
\left[
(R+P_i)\sqrt{(G^2+2\lambda\beta d)\,k}
+2^{5/4}\sqrt{\lambda dRG\,k}
+GkP_i
+\bigl(CG_cP_i+GR\bigr)\sqrt{k}
+1
\right]
+\lambda(m-1),
\]
where the $\lambda(m-1)$ term comes from the switching cost at the restart boundaries.
Since \(2^{5/4}<3\), \(\sum_{i=1}^m P_i\le P_T\), and \(m=T/k\), this implies
\[
\expec{} {\Rind_T(\comseq)}
\le
A\,\frac{T}{\sqrt{k}}
+
B\,P_T\sqrt{k}
+
GP_Tk.
\]
Because \(k\ge 1\), we have \(\sqrt{k}\le k\), and therefore
\[
\expec{} {\Rind_T(\comseq)}
\le
\phi(k)
\doteq
A\,\frac{T}{\sqrt{k}}
+
(B+G)P_T\,k.
\]

We now optimize \(\phi\) over \(k>0\). Differentiating gives
\[
\phi'(k)
=
-\frac{1}{2}AT\,k^{-3/2}
+(B+G)P_T,
\]
so the unique stationary point is
\[
\bar k
=
\left(\frac{AT}{2(B+G)P_T}\right)^{2/3}.
\]
After clipping to the admissible range \([1,T]\), define
\[
k^\star \doteq \min\{T,\max\{1,\bar k\}\}.
\]
Since \(\mathcal H\) is dyadic, there exists
\[
k^\dagger \doteq 2^{\lfloor\log_2 k^\star\rfloor}\in\mathcal H
\]
such that
\[
\frac{k^\star}{2}\le k^\dagger\le k^\star.
\]
Hence
\[
\phi(k^\dagger)
\le
\sqrt{2}\,A\,\frac{T}{\sqrt{k^\star}}
+
(B+G)P_T\,k^\star,
\]
which proves the first claim.

It remains to simplify the three regimes. If \(1\le \bar k\le T\), then
\(k^\star=\bar k\), and substituting \(\bar k\) into the display above yields
\[
\expec{} {\Rind_T(\comseq)}=
\bigo\left(A^{2/3}(B+G)^{1/3}T^{2/3}P_T^{1/3}\right).
\]
If \(\bar k<1\), then \(AT<2(B+G)P_T\), so \(k^\star=1\) and
\[
\expec{} {\Rind_T(\comseq)}\le
\sqrt{2}AT+(B+G)P_T
=
\bigo\left((B+G)P_T\right).
\]
If \(\bar k>T\), then \(AT>2(B+G)P_TT^{3/2}\), so \(k^\star=T\) and
\[
(B+G)P_TT
\le
\frac{A}{2}\sqrt{T}.
\]
Therefore,
\[
\mathbb E\left[\mathcal R_T^{\mathbf 1}(u_{1:T})\right]
\le
\sqrt{2}A\sqrt{T}+(B+G)P_TT
=
\bigo\left(A\sqrt{T}\right).
\]

Finally, if \(P_T=T^\alpha\) with \(\alpha\in[0,1]\), then
\[
P_T = T^\alpha \le T^{(2+\alpha)/3}
\qquad\text{and}\qquad
\sqrt{T}\le T^{(2+\alpha)/3},
\]
so both boundary branches are dominated by the interior rate
\[
T^{(2+\alpha)/3}
=
T^{2/3}P_T^{1/3}.
\]
This proves the final claim.
\end{proof}

\section{Dyadic intervals}
\label{sec:dyadic-intervals}

\subsection{Proof of \cref{lem:dyadic-cover}}
To prove the given expression, we need the following geometric covering fact about intervals:
\begin{lemma}\citet[Thm. 15.9]{orabona2021modern}
\label{lem:geo-intv}
Let $J=[q,r]\subseteq\mathbb{N}$ be an arbitrary interval. Then $J$ can be partitioned into two finite sequences of disjoint and consecutive intervals,
\[
(J_{-k},\dots,J_0)\subseteq \cI|_J,
\qquad
(J_1,\dots,J_p)\subseteq \cI|_J,
\]
such that
\begin{enumerate}
\item the intervals $J_{-k},\dots,J_0,J_1,\dots,J_p$ are consecutive and their union is $J$;
\item $|J_{-i}|/|J_{-i+1}|\le 1/2$ for every $i\ge 1$;
\item $|J_i|/|J_{i-1}|\le 1/2$ for every $i\ge 2$.
\end{enumerate}
\end{lemma}

\begin{proof}
Apply \cref{lem:geo-intv} to $I=[s,e]$. We obtain two sequences
\[
(J_{-k},\dots,J_0)\subseteq \cI|_I,
\qquad
(J_1,\dots,J_p)\subseteq \cI|_I,
\]
whose union is $I$. Relabeling from left to right gives a partition
$J_1,\dots,J_m\in\mathcal G$, where $m=k+p+1$.

Since every interval in $\cI$ has dyadic length, and the ratios in
\cref{lem:geo-intv} are at most $1/2$, the block lengths on each side lie on
distinct dyadic scales. As every block has length at most $L$, each side
contains at most $\lceil \log_2 L\rceil+1$ blocks. Therefore
\[
m\le 2\lceil \log_2 L\rceil+2,
\]

It remains to bound the sum of square roots. For the left side, write
\[
a_i\doteq|J_{-i}|,\qquad i=0,\dots,k,
\]
and define the tail sums
\[
S_i\doteq\sum_{j=i}^k a_j,\qquad i=0,\dots,k,
\qquad S_{k+1}\doteq0.
\]
Because the lengths decrease geometrically,
\[
S_{i+1}\le \sum_{r\ge 1}\frac{a_i}{2^r}\le a_i.
\]
Hence
\[
S_i=a_i+S_{i+1}\le 2a_i
\]
Therefore
\[
\sqrt{S_i}+\sqrt{S_{i+1}}
\le (\sqrt2+1)\sqrt{a_i},
\]
Multiplying both sides by $(\sqrt{S_i}-\sqrt{S_{i+1}})/\sqrt{a_i}$ allows us to bound $\sqrt{a_i}$ as follows:
\[
\sqrt{a_i}
\le (1+\sqrt2)\bigl(\sqrt{S_i}-\sqrt{S_{i+1}}\bigr).
\]
Summing over $i=0,\dots,k$ yields
\[
\sum_{i=0}^k \sqrt{|J_{-i}|}
\le (1+\sqrt2)\sum_{i=0}^k \bigl(\sqrt{S_i}-\sqrt{S_{i+1}}\bigr)
= (1+\sqrt2)\sqrt{S_0}.
\]
If we set
\[
A\doteq\sum_{i=0}^k |J_{-i}|,\qquad
B\doteq\sum_{i=1}^p |J_i|,
\]
then $S_0=A$, and the same argument on the right side gives
\[
\sum_{i=1}^p \sqrt{|J_i|}
\le (1+\sqrt2)\sqrt{B}.
\]
Since $A+B=L$, we conclude
\[
\sum_{r=1}^m \sqrt{|J_r|}
\le (1+\sqrt2)(\sqrt A+\sqrt B)
\le (1+\sqrt2)\sqrt{2(A+B)}
= (2+\sqrt2)\sqrt L.
\]
\end{proof}

\section{Proofs for meta-learner}
\label{sec:meta-learner-proofs}
\subsection{Proof for the one step}
\begin{proof}[Proof of \cref{lem:one-step-master-dyadic}]
Write
\[
\mathcal P_t-\mathcal P_{t-1}
=
\sum_{H\in\mathcal H} v_{t,H}\bigl(\mathcal Q_t^{(H)}-\mathcal Q_{t-1}^{(H)}\bigr)
+
\sum_{H\in\mathcal H} (v_{t,H}-v_{t-1,H})\,\mathcal Q_{t-1}^{(H)}.
\]
Taking total variation and using the triangle inequality yields
\[
\TV{\mathcal P_t-\mathcal P_{t-1}}
\le
\sum_{H\in\mathcal H} v_{t,H}\TV{\mathcal Q_t^{(H)}-\mathcal Q_{t-1}^{(H)}}
+
\TV{\sum_{H\in\mathcal H}(v_{t,H}-v_{t-1,H})\mathcal Q_{t-1}^{(H)}}.
\]
The first term is exactly
\[
\sum_{H\in\mathcal H} v_{t,H}\,c_t^{(H)}.
\]

For the second term, set
\[
a_H \doteq v_{t,H}-v_{t-1,H},
\qquad
A_+ \doteq \{H\in\mathcal H: a_H>0\},
\qquad
A_- \doteq \{H\in\mathcal H: a_H<0\}.
\]
Since both $\vec v_t$ and $\vec v_{t-1}$ lie in the simplex, $\sum_{H\in\mathcal H} a_H=0$.
Hence
\[
\alpha
\doteq
\sum_{H\in A_+} a_H
=
-\sum_{H\in A_-} a_H
=
\frac12\|\vec v_t-\vec v_{t-1}\|_1.
\]
If $\alpha=0$, there is nothing to prove. Otherwise,
\[
\sum_{H\in\mathcal H} a_H \mathcal Q_{t-1}^{(H)}
=
\alpha(\mu_+-\mu_-),
\]
where
\[
\mu_+
\doteq
\sum_{H\in A_+}\frac{a_H}{\alpha}\,\mathcal Q_{t-1}^{(H)},
\qquad
\mu_-
\doteq
\sum_{H\in A_-}\frac{-a_H}{\alpha}\,\mathcal Q_{t-1}^{(H)}
\]
are probability distributions (convex combination of probability distributions). Therefore,
\[
\TV{\sum_{H\in\mathcal H} a_H \mathcal Q_{t-1}^{(H)}}
=
\alpha\,\TV{\mu_+-\mu_-}
\le
\alpha
=
\frac12\|\vec v_t-\vec v_{t-1}\|_1.
\]
This proves \cref{eq:mixture-tv-dyadic}.

To pass from the meta dot product to the master expectation, define
\[
\vec \ell_t \doteq \bigl(\ell_t^{(H)}\bigr)_{H\in\mathcal H}\;,
\qquad
\vec c_t \doteq \bigl(c_t^{(H)}\bigr)_{H\in\mathcal H}\;,
\qquad
\vec g_t = \vec \ell_t + \lambda \vec c_t .
\]
By the lazy sampling properties (\cref{eq:lazy-marginal}), we have $\vec x_t\sim \mathcal P_t$. Also, using linearity of integration gives:

\begin{align}
\ip{\vec v_t}{\vec \ell_t}
&=
\sum_{H\in\mathcal H} v_{t,H}\,\ell_t^{(H)}
\stackrel{(\text{LOTUS})}{=}
\sum_{H\in\mathcal H} v_{t,H}\int_{\mathcal X} f_t(\vec x)\, d\mathcal Q_t^{(H)}(\vec x)
=
\int_{\mathcal X} f_t(\vec x)\,
d\Bigl(\sum_{H\in\mathcal H} v_{t,H}\mathcal Q_t^{(H)}\Bigr)(\vec x)
\\
&\stackrel{\cref{eq:master-mixture}}{=}
\int_{\mathcal X} f_t(\vec x)\, d\mathcal P_t(\vec x)
\stackrel{(\text{LOTUS})}{=}
\expec{\vec x_t\sim\cP_t}{f_t(\vec x_t)}
\label{eq:mixture-lotus}
\end{align}

Also, again by the lazy coupling identity \cref{eq:lazy-tv},
\[
\Pr(\vec x_t\neq \vec x_{t-1})
=
\TV{\mathcal P_t-\mathcal P_{t-1}}.
\]
Combining these two identities with \cref{eq:mixture-tv-dyadic} yields
\[
\E[f_t(\vec x_t)] + \lambda \Pr(\vec x_t\neq \vec x_{t-1})
\le
\ip{\vec v_t}{\vec \ell_t}
+
\lambda \ip{\vec v_t}{\vec c_t}
+
\frac{\lambda}{2}\|\vec v_t-\vec v_{t-1}\|_1
=
\ip{\vec v_t}{\vec g_t}
+
\frac{\lambda}{2}\|\vec v_t-\vec v_{t-1}\|_1.
\]
This proves \cref{eq:one-step-master-dyadic}.
\end{proof}

\subsection{The DM strongly adaptive regret result}
We use the following result from \citep{daniely2019competitive}: 
\begin{lemma}[\citep{daniely2019competitive}, Thm. 1]
\label{lem:dm-imported}
Fix a horizon $T$, a number of experts $N$, a switching cost $D\ge 0$, and an
oblivious loss sequence $l_1,\dots,l_T\in [0,1]^N$. Then there exists a randomized
online algorithm choosing experts $I_t\in[N]$ such that for every interval
$[s,e]\subseteq[T]$ of length $L=e-s+1$,
\begin{equation}
\label{eq:dm-imported}
\E\left[
\sum_{t=s}^{e} l_t(I_t)
+
D\sum_{t=s+1}^{e}\mathbf{1}\{I_t\neq I_{t-1}\}
\right]
\le
\min_{i\in[N]} \sum_{t=s}^{e} l_t(i)
+
C_{\mathrm{dm}}\sqrt{(D+1)L\log(NT)},
\end{equation}
where $C_{\mathrm{dm}}>0$ is universal.
\end{lemma}

\subsection{Algorithmic mechanism behind \citep[Thm. 1]{daniely2019competitive}}
\label{app:dm-olo-realization}

We now review the algorithmic mechanism that yields the imported guarantee of \citep[Thm. 1]{daniely2019competitive}. Since our reduction works on the simplex, we state the construction in
its simplex-valued OLO form.

A minor terminology warning is that The Daniely-Mansour mechanism itself is already hierarchical: it has a family of scale-dependent learners, together with an internal combination procedure that merges them across scales. Thus, when we invoke the
Daniely-Mansour learner as our meta learner, that meta learner has its own internal base/meta structure. We use the term \emph{scale learner} below for the internal Daniely-Mansour components to avoid confusion.

For simplicity, this section follows the constructive proof convention of \citet{daniely2019competitive}: we describe the mechanism for $T=2^J$ and $D\ge1$. This restriction is only for the constructive review below; the imported guarantee in \cref{lem:dm-imported} is stated for all $D\ge0$, and
the case $D<1$ is handled in \citet{daniely2019competitive} by reduction to $D=1$. The construction has three pieces:
\begin{enumerate}
    \item a scale learner for each dyadic scale $\tau$;
    \item a two-way combiner that softly interpolates between two simplex-valued
    algorithms;
    \item a multiscale wrapper that combines the dyadic-scale learners recursively.
\end{enumerate}
Intuitively, the scale-$\tau$ learner is tuned to intervals of length about $\tau$,
while the combiner ensures that passing from coarse to fine scales preserves control
of the switching cost.

\paragraph{Gate function.}
For parameters $\tau\ge 1$ and $Z\in(0,1/e]$, let $\tilde g_{\tau,Z}$ solve
\[
8\tilde g'_{\tau,Z}(x)=\frac{x}{\tau}\tilde g_{\tau,Z}(x)+Z,
\qquad
\tilde g_{\tau,Z}(0)=0.
\]
Define
\[
g_{\tau,Z}(x):=\Pi_{[0,1]}(\tilde g_{\tau,Z}(x)),
\qquad
U_{\tau,Z}:=\tilde g_{\tau,Z}^{-1}(1),
\]
where $\Pi_{[a,b]}$ denotes projection onto $[a,b]$.

\begin{algorithm}[t]
\caption{Fixed-Share scale learner at scale $\tau$}
\label{alg:fs-scale}
\begin{algorithmic}[1]
\Require scale $\tau$, switching cost $D\ge 1$
\State $\eta \gets \sqrt{\frac{\log(N\tau)}{D\tau}}$
\State $z_1 \gets (1/N,\dots,1/N)\in\Delta_N$
\For{$t=1,2,\dots,T$}
    \State output $z_t$
    \State observe loss vector $\ell_t\in[0,1]^N$
    \If{$\tau \ge 16D\log(N\tau)$}
        \State update
        \[
        z_{t+1}(i)=
        \frac{z_t(i)e^{-\eta \ell_t(i)}+\frac{1}{N\tau}}
        {\sum_{j=1}^N \left(z_t(j)e^{-\eta \ell_t(j)}+\frac{1}{N\tau}\right)},
        \qquad i=1,\dots,N
        \]
    \Else
        \State $z_{t+1}\gets z_t$
    \EndIf
\EndFor
\end{algorithmic}
\end{algorithm}

Algorithm~\ref{alg:fs-scale} is the scale-dependent learner used in the proof of
Theorem~2.1.

\begin{algorithm}[t]
\caption{Two-way combiner $\mathrm{Combine}_{\tau,Z}(P,Q)$}
\label{alg:two-way-combiner}
\begin{algorithmic}[1]
\Require simplex-valued OLO algorithms $P,Q$, scale $\tau$, parameter $Z\in(0,1/e]$, switching cost $D\ge 1$
\State initialize $x_1\gets 0$
\For{$t=1,2,\dots,T$}
    \State obtain current actions $p_t\in\Delta_N$ from $P$ and $q_t\in\Delta_N$ from $Q$
    \If{$\tau \ge 64D\log(1/Z)$}
        \State $\alpha_t \gets g_{\tau,Z}(x_t)$
        \State output
        $r_t \gets \alpha_t p_t + (1-\alpha_t)q_t$
    \Else
        \State output $r_t \gets p_t$
    \EndIf
    \State observe loss vector $\ell_t$
    \State let $p_{t+1},q_{t+1}$ be the next actions of $P,Q$
    \State define the surrogate losses
    \[
    \widetilde L_t(P):=\frac{\langle \ell_t,p_t\rangle + D\|p_{t+1}-p_t\|_{\mathrm{TV}}}{3},
    \qquad
    \widetilde L_t(Q):=\frac{\langle \ell_t,q_t\rangle + D\|q_{t+1}-q_t\|_{\mathrm{TV}}}{3}
    \]
    \State set
    \[
    b_t:=\frac{\widetilde L_t(P)-\widetilde L_t(Q)}{\sqrt D},
    \qquad
    x_{t+1}:=\Pi_{[-2,U_{\tau,Z}+2]}
    \left(\left(1-\frac1\tau\right)x_t+b_t\right)
    \]
\EndFor
\State \Return algorithm object $\mathcal C$ with
       $\mathcal C.\textsc{Query}(\ell_t) \to r_t$
\end{algorithmic}
\end{algorithm}

Algorithm~\ref{alg:two-way-combiner} is the key internal combination step. It mixes
the two input algorithms using the gate value $g_{\tau,Z}(x_t)$, where the scalar
state $x_t$ tracks the recent difference between their surrogate losses.

\begin{algorithm}[t]
\caption{Multiscale OLO algorithm underlying Daniely--Mansour Thm.~2.1}
\label{alg:dm-multiscale}
\begin{algorithmic}[1]
\Require horizon $T=2^J$, number of experts $N$, switching cost $D\ge 1$
\State $Z \gets \frac{1}{2T\log T}$
\For{$u=0,1,\dots,J-1$}
    \State $\tau_u \gets 2^{-u}T$
    \State $A_u \gets \textsc{Fixed-Share}(\tau_u,D)$
\EndFor
\State $B_0 \gets A_0$
\For{$u=1,2,\dots,J-1$}
    \State $B_u \gets \mathrm{Combine}_{\tau_u,Z}(B_{u-1},A_u)$
\EndFor
\State \Return $B_{J-1}$ \Comment{same interface:
       $B_{J-1}.\textsc{Query}(\ell_t)\to\vec v_t\in\Delta_N$}
\end{algorithmic}
\end{algorithm}

Algorithm~\ref{alg:dm-multiscale} maintains one scale learner for each dyadic scale
\[
T,\ T/2,\ T/4,\ \dots,\ 2,
\]
and combines them recursively from coarse to fine scales. On any interval $I$, one of
these scales is of the correct order of magnitude for $|I|$, and the recursive
combination guarantees that the final algorithm inherits that scale's regret bound up
to logarithmic overhead.
Here $J$ denotes the logarithm of the DM horizon, whereas $N$ denotes the
number of experts in the DM subproblem. In our meta-learning application,
$N=K=|\mathcal H|$.

\paragraph{Remarks.}
\begin{enumerate}
    \item %
All three algorithms share the same interface: when queried with $\ell_t\in[0,1]^N$ at round $t$, each returns a weight vector $\vec v_t\in\Delta_N$. \textsc{Combine} is a constructor that takes two such objects and produces a third. \textsc{DM-Multiscale} returns the root of the resulting recursion tree;
querying it round by round on $(\vec g_t/M)$ is exactly what \cref{alg:master} does.

    \item The denominator $3$ in Algorithm~\ref{alg:two-way-combiner} is the
    specialization of their factor $(M{+}1)$ to the case $M=2$, which is the
    value used in the proof of the multiscale combination theorem.

    \item The smallest dyadic scale is $2$, not $1$, since the construction uses
    $\tau_u=2^{-u}T$ for $u=0,\dots,J-1$.

    \item The output of Algorithm~\ref{alg:dm-multiscale} is a sequence
    $(p_t)_{t=1}^T\subseteq\Delta_N$. By the OLO/EXP equivalence in
    \citet[Sec 5.2]{daniely2019competitive}, this simplex-valued algorithm can be converted
    into a randomized experts algorithm with the same expected  loss 
     switching cost.
\end{enumerate}

Now, we can prove the exact statement of \cref{lem:meta-dyadic} based on \cref{lem:dm-imported}.
We use \cref{lem:dm-imported} in its stated form for all $D\ge0$; the
algorithmic description in \cref{app:dm-olo-realization} is included only
to recall the constructive mechanism and follows the $D\ge1$ proof
convention of \citet{daniely2019competitive}.
\begin{proof}[Proof of \cref{lem:meta-dyadic}]
We reduce our setting to the discrete expert problem by defining the number of experts $N \doteq |\mathcal{H}|$. To accommodate the bounded surrogate range $M$, we define the normalized oblivious loss sequence $l_t \doteq g_t/M \in [0,1]^N$ for all $t\in[T]$, and set the switching cost parameter to $D \doteq \lambda/M$.

Let $I_t \in [N]$ be the randomized expert chosen by the algorithm from \cref{lem:dm-imported}. We define our meta-learner's distribution over the simplex as $v_t(i) \doteq \Pr(I_t=i)$ for $i\in[N]$. Because the losses are oblivious, $v_t$ is determined purely online from the past loss history and the algorithm's internal randomness.

For each step $t$, the expected loss is exactly the inner product:
\[
\mathbb{E}[l_t(I_t)] = \sum_{i=1}^N v_t(i)l_t(i) = \langle v_t, l_t \rangle.
\]

Furthermore, the joint law of $(I_{t-1},I_t)$ is a coupling of the marginal distributions $v_{t-1}$ and $v_t$. For any coupling of two discrete distributions, the mismatch probability is lower bounded by their total variation distance. Therefore,
\[
\Pr(I_t\neq I_{t-1}) \ge \mathrm{TV}(v_t, v_{t-1}) = \frac{1}{2}\|v_t-v_{t-1}\|_1.
\]

Substituting these expected loss and mismatch probability bounds into the expectation on the left-hand side of \cref{eq:dm-imported} yields a lower bound on the discrete algorithm's cost:
\[
\sum_{t=s}^{e}\langle v_t,l_t\rangle
+
\frac{D}{2}\sum_{t=s+1}^{e}\|v_t-v_{t-1}\|_1
\le
\mathbb{E}\left[ \sum_{t=s}^{e} l_t(I_t) + D\sum_{t=s+1}^{e}\mathbf{1}\{I_t\neq I_{t-1}\} \right].
\]

By chaining this with the upper bound directly from \cref{lem:dm-imported}, we obtain the simplex formulation of the regret:
\[
\sum_{t=s}^{e}\langle v_t,l_t\rangle
+
\frac{D}{2}\sum_{t=s+1}^{e}\|v_t-v_{t-1}\|_1
\le
\min_{i\in[N]}\sum_{t=s}^{e} l_t(i)
+
C_{\mathrm{dm}}\sqrt{(D+1)L\log(NT)}.
\]

Finally, we substitute $l_t = g_t/M$ and $D = \lambda/M$ into the above inequality. Multiplying both sides by $M$ scales the square root term as follows:
\[
M \cdot C_{\mathrm{dm}}\sqrt{\left(\frac{\lambda}{M}+1\right)L\log(NT)} = C_{\mathrm{dm}}\sqrt{M^2\left(\frac{\lambda+M}{M}\right)L\log(NT)} = C_{\mathrm{dm}}\sqrt{M(M+\lambda)L\log(NT)}.
\]
This exactly recovers \cref{eq:meta-dyadic} and completes the proof.
\end{proof}

\subsection{Proof of the SAR guarantee}
\label{app:sar-proof}

\begin{proof}[Proof of \cref{thm:reduction-geometric}]
Write $J_r=[a_r,b_r]$, so $a_1=s$, $b_m=e$, and $b_r+1=a_{r+1}$ for $r<m$.

By \cref{eq:mixture-lotus}:
\[
\E[f_s(\vec x_s)] = \ip{\vec v_s}{\vec \ell_s}\le \ip{\vec v_s}{\vec g_s}.
\]
Summing the one step bound in \cref{lem:one-step-master-dyadic} over $t=s+1,\dots,e$, and using the above inequality yields
\begin{equation}
\label{eq:master-to-meta-geometric}
\E\left[\sum_{t=s}^{e} f_t(\vec x_t)
+\lambda\sum_{t=s+1}^{e}\one\{\vec x_t\neq \vec x_{t-1}\}\right]
\le
\sum_{t=s}^{e}\ip{\vec v_t}{\vec g_t}
+
\frac{\lambda}{2}\sum_{t=s+1}^{e}\|\vec v_t-\vec v_{t-1}\|_1.
\end{equation}

Since \cref{lem:meta-dyadic} holds simultaneously for every interval under the single run that generates $(\vec v_t)$, we may apply it to each block $J_r$ using the same sequence $(\vec v_t)$. 

Since each \(J_r\) is a dyadic block, define \(H_r\doteq |J_r|\in\mathcal H\).
The expert with scale \(H_r\) restarts at the beginning of \(J_r\), so on
\(J_r\) it is exactly a fresh FPRLL\((H_r)\) run.

Thus, for every $r$,
\begin{equation}
\label{eq:blockwise-meta}
\sum_{t=a_r}^{b_r}\ip{\vec v_t}{\vec g_t}
+
\frac{\lambda}{2}\sum_{t=a_r+1}^{b_r}\|\vec v_t-\vec v_{t-1}\|_1
\le
\sum_{t=a_r}^{b_r}  g_t^{(H_r)}
+
C_{\mathrm{dm}}\sqrt{M(M+\lambda)|J_r|\log(KT)}.
\end{equation}
Summing \cref{eq:blockwise-meta} over $r=1,\dots,m$ gives
\begin{align}
&\sum_{r=1}^{m}\sum_{t=a_r}^{b_r}\ip{v_t}{g_t}
+
\frac{\lambda}{2}\sum_{r=1}^{m}\sum_{t=a_r+1}^{b_r}\|v_t-v_{t-1}\|_1 \\
&\qquad\le
\sum_{r=1}^{m}\sum_{t\in J_r} g_t^{(H_r)}
+
C_{\mathrm{dm}}\sqrt{M(M+\lambda)\log(KT)}\sum_{r=1}^{m}\sqrt{|J_r|}.
\label{eq:blockwise-meta-summed}
\end{align}
Because the blocks are consecutive,
\[
\sum_{t=s+1}^{e}\|v_t-v_{t-1}\|_1
=
\sum_{r=1}^{m}\sum_{t=a_r+1}^{b_r}\|v_t-v_{t-1}\|_1
+
\sum_{r=1}^{m-1}\|v_{a_{r+1}}-v_{a_{r+1}-1}\|_1.
\]
Therefore,
\begin{align}
&\sum_{t=s}^{e}\ip{v_t}{g_t}
+
\frac{\lambda}{2}\sum_{t=s+1}^{e}\|v_t-v_{t-1}\|_1 \\
&\qquad\le
\sum_{r=1}^{m}\sum_{t\in J_r} g_t^{(H_r)}
+
C_{\mathrm{dm}}\sqrt{M(M+\lambda)\log(KT)}\sum_{r=1}^{m}\sqrt{|J_r|}
+
\frac{\lambda}{2}\sum_{r=1}^{m-1}\|v_{a_{r+1}}-v_{a_{r+1}-1}\|_1.
\end{align}
Since $\vec v_t,\vec v_{t-1}\in\Delta_N$, one has $\|\vec v_t-\vec v_{t-1}\|_1\le 2$, so each boundary term is at most $\lambda$. Thus
\begin{equation}
\label{eq:sum-blocks-bounded}
\sum_{t=s}^{e}\ip{v_t}{g_t}
+
\frac{\lambda}{2}\sum_{t=s+1}^{e}\|\vec v_t-\vec v_{t-1}\|_1
\le
\sum_{r=1}^{m}\sum_{t\in J_r} g_t^{(H_r)}
+
C_{\mathrm{dm}}\sqrt{M(M+\lambda)\log(KT)}\sum_{r=1}^{m}\sqrt{|J_r|}
+
\lambda(m-1).
\end{equation}
Combining \cref{eq:master-to-meta-geometric,eq:sum-blocks-bounded}, we obtain
\begin{align}
&\E\left[\sum_{t=s}^{e} f_t(x_t)
+\lambda\sum_{t=s+1}^{e}\one\{x_t\neq x_{t-1}\}\right] \\
&\qquad\le
\sum_{r=1}^{m}\sum_{t\in J_r} g_t^{(H_r)}
+
C_{\mathrm{dm}}\sqrt{M(M+\lambda)\log(KT)}\sum_{r=1}^{m}\sqrt{|J_r|}
+
\lambda(m-1).
\label{eq:before-base-sum}
\end{align}

By \cref{cor:static-block}, for any \(\vec u\in\X\) and every \(r\),
\begin{align}
    \sum_{t\in J_r}\bigl(g_t^{(H_r)} - f_t(\vec u)\bigr)
    \le C_{\mathrm{base}}\sqrt{|J_r|}.
\end{align}
summing over all the dyadic intervals we obtain 
\begin{align}
  \label{eq:sum-of-base-regret}
  \sum_{r=1}^{m}\sum_{t\in J_r} g_t^{(H_r)}
  \le
  \min_{u\in \X}\sum_{t=s}^{e} f_t(u)
  +
  C_{\text{base}} \sum_{r=1}^{m}  \sqrt{|J_r|}.
\end{align}

Substituting \cref{eq:sum-of-base-regret} into \cref{eq:before-base-sum}, subtracting the comparator term from both sides, and using: 
\[
\lambda(m-1)\le \lambda(2\lceil\log_2 L\rceil+1),\qquad \text{by the first part of \cref{eq:dyadic-cover}},
\]
and
\[
\sum_{r=1}^{m}\sqrt{|J_r|}\le C_{\mathrm{gc}}\sqrt{L}, \qquad \text{by the second part of  \cref{eq:dyadic-cover}},
\]
the lemma statement follows 
\end{proof}

\subsection{Proof of \cref{cor:dr-pwc}}
\label{app:sar-to-dr-pwc}

Let $\comseq\in\X^T$ be an arbitrary comparator sequence, and let
\[
S_T \;\doteq\; \sum_{t=2}^T \one\{\vec u_t\neq \vec u_{t-1}\}.
\]
This sequence induces a partition of $[T]$ into $S_T+1$ maximal consecutive segments on which $\vec u_t$ is constant,
\[
\mathcal S_i \;=\; [s_i,e_i], \qquad i=1,\dots,S_T+1,
\]
with $s_1=1$, $e_{S_T+1}=T$, and $e_i+1=s_{i+1}$ for $i=1,\dots,S_T$. Writing $\vec u^{(i)}$ for the common value of $\vec u_t$ on $\mathcal S_i$,
\begin{equation}
\label{eq:app-comparator-decomp}
\sum_{t=1}^T f_t(\vec u_t)
\;=\;
\sum_{i=1}^{S_T+1}\sum_{t\in\mathcal S_i} f_t(\vec u^{(i)})
\;\ge\;
\sum_{i=1}^{S_T+1}\min_{\vec u\in\X}\sum_{t\in\mathcal S_i} f_t(\vec u).
\end{equation}

\paragraph{Decomposing the switching cost along the partition.}
The total indicator switching cost splits into within-segment and boundary contributions:
\begin{equation}
\label{eq:app-switch-decomp}
\sum_{t=2}^T \one\{\vec x_t\neq\vec x_{t-1}\}
\;=\;
\underbrace{\sum_{i=1}^{S_T+1}\sum_{t=s_i+1}^{e_i}\one\{\vec x_t\neq\vec x_{t-1}\}}_{\text{within segments}}
\;+\;
\underbrace{\sum_{i=2}^{S_T+1}\one\{\vec x_{s_i}\neq\vec x_{s_i-1}\}}_{\text{$S_T$ boundary switches}}.
\end{equation}

Combining \cref{eq:app-comparator-decomp,eq:app-switch-decomp} and taking expectation,
\begin{align}
\E\left[\Rind_T(\comseq)\right]
&\;\le\;
\sum_{i=1}^{S_T+1}
\left(
\E\left[\sum_{t\in\mathcal S_i} f_t(\vec x_t)
+\lambda\sum_{t=s_i+1}^{e_i}\one\{\vec x_t\neq\vec x_{t-1}\}\right]
-\min_{\vec u\in\X}\sum_{t\in\mathcal S_i} f_t(\vec u)
\right) \nonumber\\
&\qquad
+\lambda\sum_{i=2}^{S_T+1}\Pr\left(\vec x_{s_i}\neq\vec x_{s_i-1}\right) \nonumber\\
&\;=\;
\sum_{i=1}^{S_T+1}\mathrm{SAR}_{\mathcal S_i}(\vec u^{(i)})
\;+\;
\lambda\sum_{i=2}^{S_T+1}\Pr\left(\vec x_{s_i}\neq\vec x_{s_i-1}\right),
\label{eq:app-dr-as-sum-sar}
\end{align}
where the last equality uses the definition of $\mathrm{SAR}$ in \cref{eq:sar-def}. Each boundary probability is at most one, so the second sum is bounded by $\lambda S_T$.

\paragraph{Applying the SAR guarantee.}
By \cref{thm:reduction-geometric}, for every segment $\mathcal S_i$,
\[
\mathrm{SAR}_{\mathcal S_i}(\vec u^{(i)})
\;=\;\tilde {\mathcal{O}}\left(\sqrt{|\mathcal S_i|}\right).
\]
Summing over $i$ and applying Cauchy--Schwarz,
\begin{equation}
\label{eq:app-cauchy-schwarz}
\sum_{i=1}^{S_T+1}\sqrt{|\mathcal S_i|}
\;\le\;
\sqrt{S_T+1}\,\sqrt{\sum_{i=1}^{S_T+1}|\mathcal S_i|}
\;=\;
\sqrt{(S_T+1)\,T}.
\end{equation}

Substituting into \cref{eq:app-dr-as-sum-sar},
\[
\E\left[\Rind_T(\comseq)\right]
\;\le\;
\tilde \bigo\left(\sqrt{(S_T+1)\,T}\right)
\;+\;
\lambda S_T,
\]
which is the claimed bound. \qed

\subsection{Proof for the DR guarantee}
\label{app:dr-proof}
\begin{proof}[Proof of \cref{thm:dynamic-dyadic-master}]
Fix any comparator sequence \(\comseq\subseteq\X\). By definition,
\[
\expec{}{\Rind_T(\comseq)}
=
\expec{}{\sum_{t=1}^T f_t(\vec x_t)}
-
\sum_{t=1}^T f_t(\vec u_t)
+
\lambda\sum_{t=2}^T \Pr(\vec x_t\neq \vec x_{t-1}).
\]
For \(t=1\), since \(c_1^{(H)}=0\) for every \(H\in\mathcal H\), we have
\(g_1^{(H)}=\ell_1^{(H)}\), and therefore
\[
\expec{}{f_1(\vec x_1)}=\langle v_1,g_1\rangle.
\]
For each \(t\ge 2\), \cref{lem:one-step-master-dyadic} gives
\[
\expec{}{f_t(\vec x_t)}
+
\lambda \Pr(\vec x_t\neq \vec x_{t-1})
\le
\langle v_t,g_t\rangle
+
\frac{\lambda}{2}\|v_t-v_{t-1}\|_1.
\]
Summing over \(t=1,\dots,T\), we obtain
\[
\expec{}{\sum_{t=1}^T f_t(\vec x_t)}
+
\lambda\sum_{t=2}^T \Pr(\vec x_t\neq \vec x_{t-1})
\le
\sum_{t=1}^T \langle v_t,g_t\rangle
+
\frac{\lambda}{2}\sum_{t=2}^T \|v_t-v_{t-1}\|_1.
\]
Subtracting \(\sum_{t=1}^T f_t(\vec u_t)\), we get
\begin{equation}
\label{eq:dynamic-master-step-1}
\expec{}{\Rind_T(\comseq)}
\le
\sum_{t=1}^T \langle v_t,g_t\rangle
+
\frac{\lambda}{2}\sum_{t=2}^T \|v_t-v_{t-1}\|_1
-
\sum_{t=1}^T f_t(\vec u_t).
\end{equation}

Now fix any \(H\in\mathcal H\), and add and subtract
\(\sum_{t=1}^T g_t^{(H)}\) on the right-hand side of
\cref{eq:dynamic-master-step-1}. This yields
\[
\begin{aligned}
\expec{}{\Rind_T(\comseq)}
&\le
\left(
\sum_{t=1}^T \langle v_t,g_t\rangle
+
\frac{\lambda}{2}\sum_{t=2}^T \|v_t-v_{t-1}\|_1
-
\sum_{t=1}^T g_t^{(H)}
\right)
\\
&\qquad\qquad+
\left(
\sum_{t=1}^T g_t^{(H)}
-
\sum_{t=1}^T f_t(\vec u_t)
\right).
\end{aligned}
\]
We now control the two brackets separately. The first is the regret of the
meta-learner relative to expert \(H\). Applying
\cref{lem:meta-dyadic} on the full interval \([1,T]\), whose length is \(T\),
gives
\[
\sum_{t=1}^T \langle v_t,g_t\rangle
+
\frac{\lambda}{2}\sum_{t=2}^T \|v_t-v_{t-1}\|_1
-
\sum_{t=1}^T g_t^{(H)}
\le
C_{\mathrm{dm}}\sqrt{M(M+\lambda)\,T\log(KT)}.
\]
Hence, for every \(H\in\mathcal H\),
\[
\expec{}{\Rind_T(\comseq)}
\le
\sum_{t=1}^T g_t^{(H)}
-
\sum_{t=1}^T f_t(\vec u_t)
+
C_{\mathrm{dm}}\sqrt{M(M+\lambda)\,T\log(KT)}.
\]
Taking the minimum over \(H\in\mathcal H\) proves
\cref{eq:dr-meta}.

It remains to instantiate the expert term. By \cref{thm:dyadic-restart}, there
exists a dyadic restart scale \(k^\dagger\in\mathcal H\) such that
\[
\sum_{t=1}^T g_t^{(k^\dagger)}
-
\sum_{t=1}^T f_t(\vec u_t)
=
\RindH{k^{\dagger}}_T
\le
\sqrt 2\,A\,\frac{T}{\sqrt{k^\star}}
+
(B+G)P_T\,k^\star.
\]
Substituting \(H=k^\dagger\) into \cref{eq:dr-meta} gives
\cref{eq:dr-final}.
\end{proof}

\section{Estimating the surrogate}
\label{subsec:hat-surrogate}

The surrogate
\[
g_t^{(H)}
=
\ell_t^{(H)}
+
\lambda c_t^{(H)}
\]
in \cref{eq:expert-losses} is convenient for analysis but is not directly
computable. Indeed,
\[
\ell_t^{(H)}
=
\E_{\vec x\sim \cQ_t^{(H)}}[f_t(\vec x)]
\]
is an expectation under the expert's sampling distribution, while
\[
c_t^{(H)}
=
\TV{\cQ_t^{(H)}-\cQ_{t-1}^{(H)}}
\]
is a total-variation distance between two high-dimensional densities.
Neither quantity has a closed form in general. We therefore use the
following implementable surrogate.

Let
\[
\epsilon_H
\doteq
\frac{\beta d}{\sigma_H}
+
\frac{\sqrt d\,G}{\mu_H}
\]
be the within-block TV bound from \cref{lem:switching-part-3} for the
scale-\(H\) parameters used by FPRLL(H). Define
\begin{align}
\label{eq:hat-c}
\hat c_t^{(H)}
&\doteq
\begin{cases}
0, & t=1,\\
\min\{\epsilon_H,1\}, & t,t-1 \text{ lie in the same block of } \cE_H,\\
1, & t \text{ is the first round of a new block of } \cE_H,
\end{cases}
\\[2mm]
\label{eq:hat-ell}
\hat\ell_t^{(H)}
&\doteq
f_t(\vec y_t^{(H)}),
\qquad
\vec y_t^{(H)}\sim \cQ_t^{(H)},
\\[2mm]
\label{eq:hat-g}
\hat g_t^{(H)}
&\doteq
\hat\ell_t^{(H)}
+
\lambda \hat c_t^{(H)}.
\end{align}
The auxiliary samples \(\vec y_t^{(H)}\) are drawn freshly for each
round and scale, after \(f_t\) is revealed, and independently of the
current meta-action conditional on the current history. The meta-learner
in \cref{alg:master} is run with
\[
\hat{\vec g}_t
\doteq
(\hat g_t^{(H)})_{H\in\cH}
\]
in place of the exact vector \(\vec g_t\).

It is useful to separate the sampled surrogate from its conditional
mean. Define
\[
\bar g_t^{(H)}
\doteq
\ell_t^{(H)}
+
\lambda \hat c_t^{(H)}.
\]
Then \(\hat g_t^{(H)}\) is an unbiased estimate of
\(\bar g_t^{(H)}\), not necessarily of \(g_t^{(H)}\). More precisely,
conditional on the history after \(f_t\) is revealed but before the
auxiliary samples are drawn,
\[
\E[\hat g_t^{(H)}]
=
\bar g_t^{(H)}.
\]
Moreover, \(\bar g_t^{(H)}\) upper-bounds the exact surrogate. Indeed,
\cref{lem:switching-part-3} gives \(c_t^{(H)}\le \epsilon_H\) inside a
restart block, while at restart boundaries we use the trivial bound
\(c_t^{(H)}\le 1\). Hence
\[
c_t^{(H)}
\le
\hat c_t^{(H)}
\qquad\text{and therefore}\qquad
g_t^{(H)}
\le
\bar g_t^{(H)}.
\]

We now explain why using \(\hat{\vec g}_t\) does not change the guarantees.
The original analysis uses the surrogate vector in three places.

\paragraph{1. The one-step master reduction.}
The original one-step bound is
\[
\E[f_t(\vec x_t)]
+
\lambda\Pr(\vec x_t\neq \vec x_{t-1})
\le
\langle \vec v_t,\vec g_t\rangle
+
\frac{\lambda}{2}\|\vec v_t-\vec v_{t-1}\|_1 .
\]
Since \(g_t^{(H)}\le \bar g_t^{(H)}\) for every \(H\) and
\(\vec v_t\in\Delta_K\),
\[
\langle \vec v_t,\vec g_t\rangle
\le
\langle \vec v_t,\bar{\vec g}_t\rangle .
\]
Taking expectation over the auxiliary samples gives
\[
\E[\langle \vec v_t,\hat{\vec g}_t\rangle]
=
\E[\langle \vec v_t,\bar{\vec g}_t\rangle].
\]
Therefore,
\[
\E[f_t(\vec x_t)]
+
\lambda\Pr(\vec x_t\neq \vec x_{t-1})
\le
\E[\langle \vec v_t,\hat{\vec g}_t\rangle]
+
\frac{\lambda}{2}\E[\|\vec v_t-\vec v_{t-1}\|_1].
\]
Thus the real one-step cost of the master is still controlled, in
expectation, by the sampled surrogate losses fed to the meta-learner.

\paragraph{2. The meta-regret bound.}
The meta-regret lemma is applied directly to the realized sequence
\[
\hat{\vec g}_1,\ldots,\hat{\vec g}_T.
\]
This sequence is valid full-information feedback for the meta-learner:
\(\hat{\vec g}_t\) is observed after the meta-action is chosen, and it is
generated independently of the current meta-action conditional on the
history. Also, by boundedness of the losses,
\[
0\le \hat g_t^{(H)}\le M_f+\lambda \doteq M.
\]
Therefore \cref{lem:meta-dyadic} applies with \(\hat{\vec g}_t\). For
any interval \(I=[s,e]\) of length \(L\),
\[
\sum_{t=s}^{e}
\langle \vec v_t,\hat{\vec g}_t\rangle
+
\frac{\lambda}{2}
\sum_{t=s+1}^{e}
\|\vec v_t-\vec v_{t-1}\|_1
\le
\min_{H\in\cH}
\sum_{t=s}^{e}
\hat g_t^{(H)}
+
C_{\mathrm{dm}}\sqrt{M(M+\lambda)L\log(KT)}.
\]
Taking expectations and using
\[
\E[\min_H Z_H]\le \min_H \E[Z_H],
\]
we obtain
\[
\E\!\left[
\sum_{t=s}^{e}
\langle \vec v_t,\hat{\vec g}_t\rangle
+
\frac{\lambda}{2}
\sum_{t=s+1}^{e}
\|\vec v_t-\vec v_{t-1}\|_1
\right]
\le
\min_{H\in\cH}
\sum_{t=s}^{e}
\left(
\ell_t^{(H)}
+
\lambda\hat c_t^{(H)}
\right)
+
C_{\mathrm{dm}}\sqrt{M(M+\lambda)L\log(KT)}.
\]

\paragraph{3. The best-expert term after the meta-regret bound.}
After applying the meta-regret bound, the remaining quantity is the cost
of the best scale \(H\):
\[
\sum_t
\left(
\ell_t^{(H)}
+
\lambda\hat c_t^{(H)}
\right).
\]
This differs from the original exact surrogate cost
\[
\sum_t g_t^{(H)}
=
\sum_t
\left(
\ell_t^{(H)}
+
\lambda c_t^{(H)}
\right)
\]
only in the switching term. The replacement is harmless because
\(\hat c_t^{(H)}\) is chosen to match the switching upper bounds already
used in the base-expert analysis. Inside each restart block,
\[
\hat c_t^{(H)}
=
\min\{\epsilon_H,1\}
\le
\epsilon_H,
\]
which is the same within-block movement bound used in
\cref{cor:static-block} and \cref{thm:dyadic-restart}. At a restart
boundary,
\[
\hat c_t^{(H)}=1,
\]
so the contribution is at most \(\lambda\), which is absorbed by the
existing restart-boundary terms.

Consequently, the arguments of \cref{cor:static-block} and
\cref{thm:dyadic-restart} give the same asymptotic bounds for
\[
\sum_t
\left(
\ell_t^{(H)}
+
\lambda\hat c_t^{(H)}
\right)
-
\sum_t f_t(\vec u_t)
\]
as for the original exact surrogate regret, up to universal constant
changes in the existing restart-boundary terms.

Combining the three observations, the guarantees of
\cref{thm:reduction-geometric}, \cref{cor:dr-pwc}, and
\cref{thm:dynamic-dyadic-master} continue to hold in expectation for the
algorithm using \(\hat{\vec g}_t\), with the same rates.

\end{document}